\documentclass[twoside]{article}

\usepackage[accepted]{aistats2023}
\usepackage[round]{natbib}

\usepackage{booktabs}       

\usepackage{caption}

\usepackage{algorithmic}
\usepackage{algorithm}

\usepackage{amsmath}
\usepackage{amssymb}
\usepackage{mathtools}
\usepackage{amsthm}

\usepackage{accents}

\DeclareMathOperator*{\argmax}{arg\,max\;\;}
\DeclareMathOperator*{\argmin}{arg\,min\;\;}

\usepackage{quiver}
\usepackage{listings}
\usepackage{xcolor}
\definecolor{mygreen}{rgb}{.05,.4,.05}
\definecolor{myblue}{rgb}{.05,.05,.4}

\lstdefinelanguage{Julia}%
  {morekeywords={abstract,struct,function,end,using,return,for,in},%
   sensitive=true,%
   morestring=[s]{"}{"},%
}[keywords,comments,strings]%

\DeclareMathOperator{\Tr}{Tr}

\newcommand\munderbar[1]{%
  \underaccent{\bar}{#1}}

\newcommand\fakesection[1]{%
 \paragraph{#1}}

\renewcommand{\algorithmiccomment}[1]{\hfill#1}

\newtheorem{problem}{Problem}

\theoremstyle{plain}
\newtheorem{theorem}{Theorem}[section]

\newtheorem{lemma}[theorem]{Lemma}
\newtheorem{corollary}[theorem]{Corollary}
\theoremstyle{definition}

\theoremstyle{remark}

\begin{document}

%
\runningtitle{Dimensionality Collapse: Optimal Measurement Selection for Low-Error Infinite-Horizon Forecasting}

%

\twocolumn[

\aistatstitle{Dimensionality Collapse: \\ 
                Optimal Measurement Selection for Low-Error Infinite-Horizon Forecasting}

\aistatsauthor{ Helmuth Naumer \And Farzad Kamalabadi }

\aistatsaddress{ University of Illinois Urbana-Champaign \And University of Illinois Urbana-Champaign} ]

\begin{abstract}
    This work introduces a method to select linear functional measurements of a vector-valued time series optimized for forecasting distant time-horizons.
    By formulating and solving the problem of sequential linear measurement design as an infinite-horizon problem with the time-averaged trace of the Cram\'{e}r--Rao lower bound (CRLB) for forecasting as the cost, the most informative data can be collected irrespective of the eventual forecasting algorithm.
    By introducing theoretical results regarding measurements under additive noise from natural exponential families, we construct an equivalent problem from which a local dimensionality reduction can be derived.
    This alternative formulation is based on the future collapse of dimensionality inherent in the limiting behavior of many differential equations and can be directly observed in the low-rank structure of the CRLB for forecasting.
    Implementations of both an approximate dynamic programming formulation and the proposed alternative are illustrated using an extended Kalman filter for state estimation, with results on simulated systems with limit cycles and chaotic behavior demonstrating a linear improvement in the CRLB as a function of the number of collapsing dimensions of the system.
\end{abstract}

\section{INTRODUCTION}
\label{sec:intro}
There is a long history of optimal experimental design in statistics.
In such work, rather than directly design estimators, sequences of experiments are designed to optimize some information criteria.
This criteria in turn implies the existence of accurate estimators.
These optimization objectives often take the form of functions of the inverse Fisher information matrix such as the alphabet optimality criteria including minimizing the trace, determinant, or maximum eigenvalue of the matrix \citep{chaloner_1995, Dette_designing}.
In the past two decades, a significant motivation of optimal experimental design has been found in optimal sensor placement \citep{martinez_2006, ranieri_2014, xygkis_2018}.
In control theory, sequential experimental design appears in the design of informative inputs to a parameterized system for the purpose of identifying the unknown parameters \citep{fohring_adaptive, asprey_designing, Huan_sequential, titterington_aspects}.
Despite this history, it has recently been observed that optimizing for correct system parameters often produces poor forecasts of the future system state, and that predictive performance should be considered directly \citep{casey_optimal, Transtrum_2012, letham_2016}.
These works typically assume the ability to control the system itself, observing multiple full trajectories and driving the system into informative regions of the state space.
Contrary to this, in many time-series forecasting problems, we only observe a single trajectory of which we have no control.

Time series forecasting is a long-studied problem often motivated by numerous applications such as weather prediction \citep{penland_prediction_1993, Ehrendorfer_1994, ehrendorfer_2006}, finance \citep{Abu-Mostafa_1996, Cao_2001}, and control \citep{Bertsekas_2017_1}.
Oftentimes, algorithms focus only on short-term predictions, intrinsically coupled with state estimation as seen in classical techniques such as the Kalman filter and particle filter.
In the extreme short-term, there have been significant developments in recent years applying machine learning tools such as generative adversarial networks to form the predictions \citep{shi_2017,Ravuri_2021}.

On the other extreme timescale, there has been growing interest in reconstructing information about the limiting behavior of the time series from data.
The approaches vary dramatically from the application of embeddings theorems to reconstruct the shape of attractors from lower-dimensional measurements \citep{Gilphi_2020}, to regression techniques that learn sparse governing equations \citep{Brunton_2016,Kaiser_2018,Zhang_2019,messenger_2021}, to Koopman operator and Dynamic Mode Decomposition (DMD) based techniques which form inherently linear approximations in higher dimensional spaces \citep{schmid_2010,williams_2015,proctor_2016,erichson_2019}.
Furthermore, the scientific computing community has been recently experimenting with using neural networks to approximate the solutions of differential equations through physics-informed neural networks \citep{raissi_2019} as a more computationally efficient alternative to Gaussian process regression \citep{raissi_2018}.

Even when the underlying data is not governed by differential equations, many modern techniques seek to learn a differential equation which well-represents the phenomena of interest as a new general form of regression --- enabling the application of dynamical systems tools to broader learning problems.
While some of these techniques are focused on using the expressive nature of differential equations in lieu of neural networks \citep{chen_2018,dupont_2019,dandekar_2020,massaroli_2020,avelin_2021}, others instead use the extensive understanding of differential equations to constrain the geometry of the estimates based on integrals of motion \citep{greydanus_2019,cranmer_2020}.
Furthermore, there have been attempts to construct implicitly deep networks by fitting the fixed points of an ODE in deep equilibrium models \citep{bai_2019, bai_2020, pabbaraju_2021,gilton_2021}.

A fortunate consequence of this trend in modeling is that the generative assumptions involving differential equations often implicitly imbue the future with significant structure.
Even Lyapunov stability, one of the weakest forms, immediately implies that the feasible set of future points monotonically shrinks over time \citep{khalil_2002}.
In many cases though, this additional structure takes the form of the guarantee that the long-term behavior is actually confined to a lower dimensional space than the ambient dimensionality.
In this work, we propose methods based on this dimensionality collapse for informed data collection in forecasting time series modeled by differential equations.


Our key observation is that, by incorporating the future structure of the dynamical system into the present, many linear measurements become entirely uninformative --- no meaningful prediction regarding distant time-horizons can be made from the observation.
By avoiding these uninformative measurements in favor of observing the informative subspace, we introduce a method for reducing the Cram\'{e}r--Rao lower bound (CRLB) for forecasting.
This reduction scales linearly with the dimensionality of the uninformative space, reducing the errors in estimation resulting from embedding a low-dimensional manifold of feasible data in a high-dimensional ambient observation space.

\subsection*{Contributions}
\begin{itemize}
    \item We analyze the problem of sequential linear measurement design in systems governed by differential equations with low-dimensional limiting behavior, providing both theoretical results and practical algorithms.
    \item We prove that the time-averaged infinite horizon optimal sampling problem is equivalent to maximizing the sum of $K$ squared inner products subject to a norm constraint, where $K$ is the dimensionality of the limiting behavior, thereby reducing the dimensionality of the problem.
    \item We specialize our results to construct a computationally efficient solution to the problem when the system converges to an isolated limit cycle.
    \item We demonstrate the algorithm on simulated chaotic systems and limit cycles systems: illustrating the broad applicability of the model assumptions and the linear improvement of the CRLB as a function of the number of collapsing dimensions.
\end{itemize}

\section{PROBLEM FORMULATION}
\label{sec:model}
In this work, we construct a method to optimize a policy for linear functional sampling which minimizes the CRLB for forecasting future values of the time series.
There are many closely related variations of this problem, all built upon the same theme.
We begin this section by defining the model assumptions, before explicitly defining the variations of the problem solved in the manuscript.
The section concludes with a brief note on the dynamic programming approach to the problems, which represent a clear baseline.

\begin{table*}[t]
    \vskip 0.15in
    \begin{center}
    \begin{small}
    \begin{sc}
    \begin{tabular}{|c|ccc|}
    \toprule
    System   & Parameter Conditions & Ambient Space & Limiting Space \\
    \midrule
    Gradient Flow & Morse Function & $\mathbb{R}^M$  & Countable Set \\
    Van Der Pol             & Known $\mu>0$ & $\mathbb{R}^2$ & $\mathbb{S}^1$ \\
    Hopf Bifurcation        & Unknown $\alpha, \lambda$  & $\mathbb{R}^4$ & $\mathbb{R}_+ \times \mathbb{S}^1$ \\
    Lorenz                  & Known $\sigma, \rho$  & $\mathbb{R}^3$ & Fractal Dimension 2.06 \\
    Heat Equation           & Known $\alpha>0$ & $W^{1,2}(\Omega)$      & $\mathbb{R}$ \\
    \bottomrule
    \end{tabular}
    \end{sc}
    \end{small}
    \end{center}
    \vskip -0.1in
    \caption{Example dynamical systems with larger ambient space than limiting space.}
    \label{tab:examples}
\end{table*}

\fakesection{System Model}
This work analyzes the state space model of the form
\begin{align}
    \dot{x} &= f(x),\label{eq:ssdyn}\\
    y_i &= \langle u_i, x_{\tau_i} + \xi_i \rangle \label{eq:ssmodel}
\end{align} 
where $x \in \mathbb{R}^M$ is the state with temporal derivative $\dot{x} = \frac{dx}{dt}$, $f: \mathbb{R}^M \rightarrow \mathbb{R}^M$ is Lipschitz on some domain containing the system trajectory with a Lipschitz continuous derivative, $y_i$ is the $i$'th measurement, $x_{\tau_i}$ is the state at time $\tau_i$, $\langle u_i, \cdot \rangle$ represents the $i$'th chosen measurement functional based on the Euclidean inner product, and $\xi_i$ is additive noise during observation $i$, independent between timesteps with further assumptions to be described in detail in the following paragraph.
In general, we can combine unknown constant parameters of the dynamical system with the state.
The Lipschitz dynamics constraint is common when working with differential equations because it enables the existence and uniqueness theorem, guaranteeing a unique solution to a differential equation for a given initial state \citep{khalil_2002}.
The uniqueness of trajectories immediately implies the existence of a set of time-advancing functions $\varphi^\tau: \mathbb{R}^M \rightarrow \mathbb{R}^M$ which satisfies the semigroup property $\varphi^{\tau_1} \circ \varphi^{\tau_2} = \varphi^{\tau_1 + \tau_2},$ where $\circ$ represents function composition.
Such a group of functions is known as the flow of the dynamical system.
We denote the state at time $\tau$ as $x_\tau$, and the flow is defined such that $x_\tau = \varphi^\tau(x_0)$ for any given initial state $x_0$.
If additionally $f$ is differentiable in $x$, then $\varphi^\tau$ can be directly shown to be continuously differentiable in $x$ for all $\tau$ through the construction of the sensitivity equation \citep{khalil_2002}.
Differentiability is only a very mild additional assumption for Lipschitz dynamical systems, as Lipschitz functions from $\mathbb{R}^N$ onto $\mathbb{R}^N$ can always be uniformly approximated by smooth Lipschitz functions [see, for instance, \citep{hajek_2010} for a more general case].

\fakesection{Measurement Family}
The family of measurements is defined through two steps.
First, we assume that the state of the system represents some physically meaningful quantities for which we have some method of controlled linear functional observation under additive noise, as shown in Equation \eqref{eq:ssmodel}.
Linear functional observations capture many common types of measurements.
If our state space is extended, for instance, to almost everywhere bounded continuous functions, i.e. $(C(\Omega),\|\cdot\|_{\infty})$, then line integrals representing tomography and Dirac delta functions representing discrete sensor placement both exist in our measurement family.
Furthermore, in finite dimensional spaces, linear combinations of measurements represent common applications such as beamforming or discretized versions of the previous continuous examples.

We further assume that each $\xi_i$ is in a natural exponential family (NEF).
A random variable $\xi$ is said to be in a NEF if the probability density function can be written as
\begin{equation}
    p(\xi; \theta) = \exp\{\theta^\top \xi - A(\theta) + B(\xi)\}
\end{equation}
where $A$ is the log-partition function, $\theta$ is an unknown parameter vector, and $B$ is the log-base function.
Many common distributions can be reparameterized into this form, such as Gaussian, Gamma, or Poisson distributions \citep{Morris_2006}. 
\textit{Note that the elements of the noise vector need not come from the same probability distribution, nor even the same family of distributions.}
For instance, one element could represent Gaussian noise, while another could be Poisson noise.

\fakesection{Forecasting Bound Objective}
The CRLB is a well-known matrix inequality for the covariance of unbiased estimators. 
While a full description is provided in Appendix \ref{app:crlb}, we briefly describe the key properties in this section.
Through the well-known decomposition of mean squared error (MSE) into the sum of squared bias and variance, $\text{MSE} = \mathbb{E}\left[\|\hat{x} - \mathbb{E}[\hat x]\|_2^2\right]$ for all $\hat{x}$ such that $\mathbb{E}[\hat{x} - x] = 0$.
By expressing the squared norm as $\|\hat{x} - \mathbb{E}[\hat x]\|_2^2 = (\hat{x} - \mathbb{E}[\hat x])^\top(\hat{x} - \mathbb{E}[\hat x])$ and applying the cyclic permutation invariance and linearity of the trace, we observe that the MSE is lower bounded as
\begin{align}
    \mathbb{E}[\|\hat{x} - x\|_2^2] &= \text{Tr}\left(\mathbb{E}[(\hat{x} - \mathbb{E}[\hat{x}])(\hat{x} - \mathbb{E}[\hat{x}])^\top ]\right)\\
    &= \text{Tr}\left(\Sigma_{\hat{x}}\right) \geq \text{Tr}(\munderbar{\Sigma}),
\end{align}
where $\Sigma_{\hat{x}}$ is the covariance of $\hat{x}$ and $\munderbar{\Sigma}$ is the CRLB, and $\Tr$ denotes the trace of a matrix.
By considering the future state to be the unknown parameter, the same bound may be applied to forecasting.
In forecasting deterministic Lipschitz dynamical systems based on measurements in exponential families of random variables, the CRLB undergoes the linear transformation
\begin{equation}
    \munderbar{\Sigma}_{\tau} = \left[d \varphi^\tau_{x_0}\right] \munderbar{\Sigma}_{0} \left[d \varphi^\tau_{x_0}\right]^\top,\label{eq:reparam}
\end{equation}
where $\munderbar{\Sigma}_\tau$ denotes the CRLB for predicting the state at time $\tau$, $\munderbar{\Sigma}_0$ denotes the CRLB for estimating the current state, and $d \varphi^\tau_{x_0}$ represents the differential of $\varphi^\tau$ at $x_0$.
Brackets have been included to emphasize that $d \varphi^\tau_{x_0}$ is a matrix and represents the Jacobian of the flow.
This transformation comes directly from the observation that advancing the state forward in time is a diffeomorphic reparameterization of the probability distribution \citep{lehmann_1998}.

We assert that in many stable systems, $d \varphi^\tau_{x_0}$ becomes approximately low-rank as $\tau$ grows large, and thus we can consider optimal data collection to be done by making measurements which minimize $\munderbar{\Sigma}_0$ along the non-zero modes of the Jacobians.
This low-rank structure occurs due to systems converging to lower-dimensional limiting spaces.
In Table \ref{tab:examples}, we provide a number of examples of dynamical systems with their ambient space and limiting space.
In the most extreme case, the heat equation reduces from a Sobolev space $W^{1,2}(\Omega)$, which is countably infinite, to the real line which is one dimensional.
Full system definitions are provided in Appendix \ref{app:table}.

\fakesection{Measurement Optimization}
Representing the expected squared magnitude of the error of potential estimators, we consider our cost to be the trace of the CRLB.
We consider both the time averaged cost which is applicable to problems such as limit cycles, as well as a discounted cost formulation more strongly applicable to chaotic dynamics due to the exponential growth of the CRLB based on the largest Lyapunov exponent \citep{Bertsekas_2017, strogatz_2015}.

To simplify the computational aspects while retaining the meaningful structure, we approximate our cost through sampling.
Then, we have the following two infinite-horizon dynamic programming problems, where $\gamma$ represents the discount factor and $\mathcal{S}$ represents the set of indices for the measurement times.
\begin{problem}[Discounted Cost Optimal Design]
    \begin{equation}
        \argmin_{\{u_i\}_{i \in \mathcal{S}}}\lim_{N \rightarrow \infty} \sum_{j=1}^{N} \gamma^{j} \Tr \left(\munderbar{\Sigma}_{\tau_j}(\{u_i\}_{i \in \mathcal{S}})\right) \label{eq:discount}
    \end{equation}
\end{problem}

\begin{problem}[Average Cost Optimal Design]
    \begin{equation}
        \argmin_{\{u_i\}_{i \in \mathcal{S}}} \lim_{N \rightarrow \infty} \frac{1}{N} \sum_{j=1}^{N} \Tr \left( \munderbar{\Sigma}_{\tau_j}(\{u_i\}_{i \in \mathcal{S}})\right)
    \end{equation}
\end{problem}

Both discounted and average cost infinite-horizon problems are well-studied in stochastic control and in reinforcement learning, where numerous variants of value and policy iteration are applied \citep{Bertsekas_2017}.

An important note is that in unstable and chaotic systems, we still observe an approximate low dimensionality.
Consider a system for which the two largest eigenvalues of $d\varphi^{\tau}$ are $\lambda_1, \lambda_2 > 1$.
Without additional sampling, the ratio of the two largest eigenvalues of the geometric sum $\sum_j \gamma^j d\varphi^{\tau j}$ in the discounted cost formulation will become
\begin{equation}
    \frac{1 - \lambda_2 \gamma}{1 - \lambda_1 \gamma}.
\end{equation}
There always exists a discount factor $\gamma<1$ for which this gap becomes arbitrarily large.
Hence, there will always be some discounted cost formulation for which the cost matrix corresponding to the CRLB is approximately singular.

\fakesection{Dynamic Programming Formulation}
The standard formulation to infinite-horizon discounted and average cost problems is through the Bellman equation, enabling common algorithms such as value iteration and policy iteration.
For a known state, the discounted cost Bellman equation becomes
\begin{equation}
    \begin{split}
        x_+ &= \varphi^\tau(x) \\
        \munderbar{\Sigma}_+(u) &= \left[d \varphi^\tau_x\right]\left(\munderbar{\Sigma}^{-1} + \|u\|_{\Lambda_\sigma}^{-2} u u^\top\right)^{-1}\left[d \varphi^\tau_x\right]^\top\\
        J(x, \munderbar{\Sigma}) &= \Tr \munderbar{\Sigma} + \gamma \max_{u \in \mathcal{A}}\left\{J\left(x_+,  \munderbar{\Sigma}_+(u)\right)\right\} ,
    \end{split}
    \label{eq:update}
\end{equation}
where $x_+$ and $\munderbar{\Sigma}_+$ represent the updated state and CRLB respectively, $\|u\|_{\Lambda_\sigma}^2 = u^\top \Lambda_\sigma u$ represents the squared norm induced by the measurement noise covariance matrix $\Lambda_\sigma$, $J$ represents the value function, and $\mathcal{A}$ denotes the set of possible normalized measurements functionals \citep{Bertsekas_2017}. 
The update to the CRLB comes from the total Fisher information of two independent observations being the sum of the Fisher information matrices.
The CRLB update in Equation \eqref{eq:update} comes directly from the form of the Fisher Information in Theorem \ref{thm:rank}.

To apply most common iterative algorithms for solving the Bellman equation, the state space must be discretized.
Without strengthening the assumptions beyond the typical, already asserted, dynamical systems assumptions, the number of required samples can be shown to grow exponentially with the dimensionallity of the system.
As a baseline in this work, we include a local-averaging approximation, the details of which including sample complexity are available in Appendix \ref{app:approx} of the supplemental information.

\section{DIMENSIONALITY REDUCTION TECHNIQUE}
\label{sec:oracle}

The computational cost of the dynamic programming approach with sampled state space approximation becomes prohibitively expensive as the ambient dimensionality of the system grows (See Appendix \ref{app:approx}).
However, the future of a dynamical system will often be low dimensional, and in many cases approximately one dimensional.
In this section, we analyze the problem when the governing dynamical system converges to a low dimensional limit set, enabling a computationally efficient alternative to the dynamic programming formulation for the average cost formulation.
We then briefly discuss the extension to arbitrary low-dimensional embedded limit sets.
While proofs for all results appear in Appendix \ref{app:oracle}, we include numerous lemmas to aid in the exposition.

The core ideas of this section can be seen in Figure \ref{fig:DimensionCollapse} applied to the Van Der Pol oscillator, a common example of a system exhibiting a limit cycle.
In all panels, the plot axes represent the state of the 2D system.
In panels A and B, we see the vector field and some example trajectories of system.
Panel D illustrates $\varphi^\infty := \lim_{\tau \rightarrow \infty} \varphi^\tau$ computed numerically, where the color indicates a representation of the phase along the limit cycle.
As the mapping $\varphi^\infty$ illustrated in panel D maps onto a one-dimensional space in a smooth manner outside of the origin, the Jacobian of the limiting flow in the ambient space is at most rank 1.
The row space of $d \varphi^\infty_x$ is plotted in panel C, representing the subspace that does not vanish as time goes to infinity.
Our proposed method seeks to minimize the CRLB with respect to this subspace.

\begin{figure}[ht]
    \centering
    \includegraphics[width=.45\textwidth]{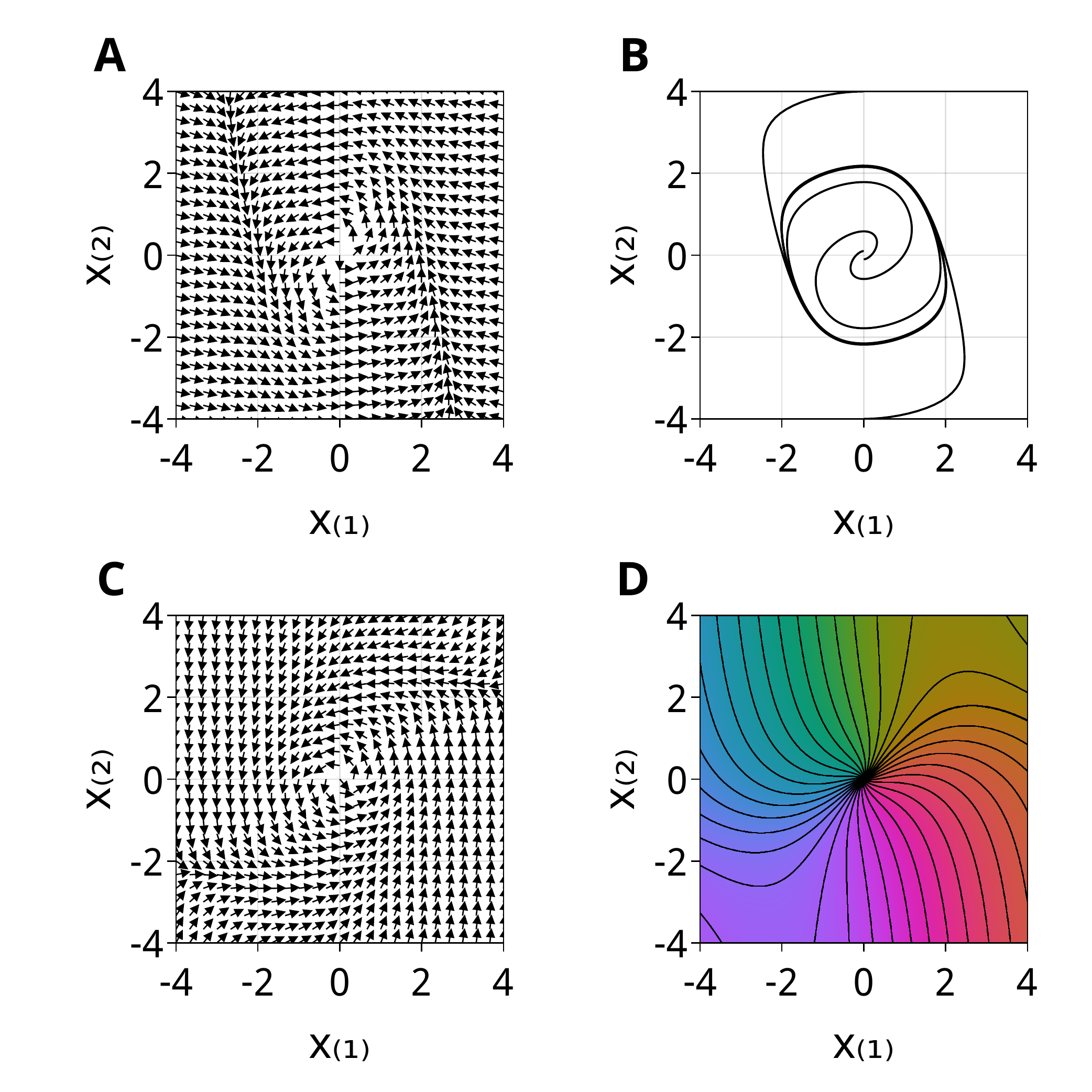}
    \vspace*{-3mm}
    \caption{A: Vector field of Van Der Pol oscillator; B: Four example trajectories;  C: Row space of limiting Jacobian; D: A smooth map $\varphi^\infty: (\mathbb{R}^2 - \{0\}) \rightarrow \mathbb{S}^1$ defined by the dynamics.}
    \label{fig:DimensionCollapse}
\end{figure}

In Section \ref{sec:fish}, we introduce a series of statements characterizing the Fisher information matrix of the state vector for a single measurement of the form described in Equation \eqref{eq:ssmodel}.
We show that the contribution of each measurement is, under a mild condition, the outer product of the measurement vector $u$, normalized by a norm defined by the noise distribution covariance matrix.
Then, in Section \ref{sec:flow} we show that this matrix structure enables a simple reformulation of the infinite horizon objective in order to simplify the problem of sequential optimal experimental design for forecasting.

\subsection{Linear Measurements And Fisher Information}
\label{sec:fish}
In this section, we show that independent NEFs preserve essential structures under a weighted average, resulting in a simple expression for the Fisher information of a linear functional measurement.

First, we must observe that the weighted average of independent random variables in NEFs is itself in another NEF.
\begin{lemma}[Proof: See Lemma \ref{prf:nef}]
    \label{lemma:nef}
    Given a finite set of independent random variables $\{\xi_{(i)}\}$ in NEFs with log-partition functions $\{A_i\}$, parameters $\{\theta_i\}$, and a set of real-valued weights $\{\alpha_i\}$, then $\bar \xi = \sum_i \alpha_i \xi_{(i)}$ is in a NEF with a log partition function
    \begin{equation}
        A(\theta) := \sum_i A_i\left(\theta +  \alpha_i\theta_i\right).
    \end{equation}
\end{lemma}

While we still lack a significant amount of knowledge regarding the specific distribution, fundamental properties of the log-partition function allow us to directly compute the Fisher information of the observation about the parameter in the new family of probability distribution.
\begin{lemma}[Proof: See Lemma \ref{prf:fisher}]
    \label{lemma:fisher}
The Fisher information of a random variable in a NEF is the variance of the random variable.
\end{lemma}

\begin{figure}[t]
\[\begin{tikzcd}
	{\mu} & {\bar \mu} \\
	{\theta} & {\bar \theta}
	\arrow["h", from=1-1, to=1-2]
	\arrow["{g^{-1}}", from=1-2, to=2-2]
	\arrow["{\{g_i^{-1}\}}"', from=1-1, to=2-1]
	\arrow["h"', from=2-1, to=2-2]
	\arrow["T"{description}, dashed, from=1-1, to=2-2]
\end{tikzcd}\]
    \caption{Commutative diagram of key maps and variables.}
    \label{fig:commute}
\end{figure}
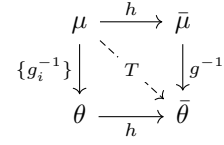

The remaining task is to relate the Fisher information about the new parameter to the Fisher information regarding the mean values of the original distributions.
The connections between the key parameters of interest are illustrated in the commutative diagram shown in Figure \ref{fig:commute}.
In the illustration, $\mu$ represents the mean of the noise vector $\xi$, $\bar \mu$ represents the mean of $\bar \xi$, $\theta$ represents the parameter vector of $\xi$, and $\bar \theta$ represents the parameter of $\bar \xi$ in the new model family.
We can immediately observe $h: \mu \mapsto \langle u, \mu \rangle$ by construction, and $g: \theta \mapsto \left.\frac{\partial}{\partial \Theta} A(\Theta) \right|_{\Theta = \theta}$ with the inverse denoted $g^{-1}$.
Next, $g_i$ and $g$ are defined identically, but for $\xi_{(i)}$ rather than $\bar \xi$.
The diagram commutes, and so we denote $T$ to be the implicitly defined mapping from $\mu$ to $\bar \theta$.

\begin{lemma}[Proof: See Lemma \ref{prf:diagram}]
    \label{lemma:diagram}
    The diagram in Figure \ref{fig:commute} commutes for any vector $\xi$ of independent random variables from NEFs.
\end{lemma}

Before using the above quantities, we need to establish when $g^{-1}$ exists.
While full proofs are available in the appendix, invertibility comes from proving $g$ is strictly monotonic.

\begin{lemma}[Proof: See Lemma \ref{prf:inverse}]
    \label{lemma:inverse}
    Assume the log-partition function $A$ of a NEF is twice differentiable and that the second moment of the associated random variable exists.
    Then the inverse of $g(\theta) := \left.\frac{\partial}{\partial \Theta} A(\Theta) \right|_{\Theta = \theta}$ exists if and only if the random variable is non-degenerate for almost all $\theta$.
\end{lemma}

Finally, observe that $T$ is an injective function, and is thus invertible on the image.
Furthermore, while $g_i^{-1}$ is model specific, the Jacobian is directly related to the Fisher information.
By recalling that Jacobian of $g_i$ is the Fisher information of the probability distribution in the family, we can conclude the following theorem as the result of a reparameterization by $T^{-1}$.

\begin{theorem}[Proof: See Theorem \ref{prf:rank}]
    \label{thm:rank}
    Suppose $\xi$ is a vector of independent random variables from NEFs with twice differentiable log-partition functions, mean vector $\mu$, and a diagonal covariance matrix $\Lambda_\sigma$.
    Then the Fisher information of the observation $\langle u, \xi\rangle$ is 
    \begin{equation}
        I(\mu) = \frac{u u^\top}{u^\top \Lambda_\sigma u}
    \end{equation}
\end{theorem}

As $\|u\|_{\Lambda_\sigma} = \sqrt{u^\top \Lambda_\sigma u}$ is a valid norm, we additionally immediately observe that the Fisher information is invariant to the scaling of the measurement functional.
This corollary allows us to restrict our search for informative measurements to unit spheres.
\begin{corollary}[Proof: See Corollary \ref{prf:invariant}]
    \label{cor:invariant}
    The Fisher information for state estimation under independent noise is invariant to the scaling of the linear functional.
\end{corollary}

Finally, when the noise is identically distributed, the form of the Fisher information matrix simplifies further.
\begin{corollary}[Proof: See Corollary \ref{prf:iid}]
    \label{cor:iid}
    Suppose the elements of the noise vector are i.i.d. with variance $\sigma^{2}$.
    Then
    \begin{equation}
        I(\mu) = \sigma^{-2} \frac{u u^\top}{\|u\|_2^2}
    \end{equation}
\end{corollary}

\subsection{Transporting the CRLB Through Time}
\label{sec:flow}
We now use the form of the Fisher information matrix in Theorem \ref{thm:rank} to construct an alternative optimization objective in Theorem \ref{thm:weakened} with the same solution as the original infinite-horizon optimal experimental design problem.
We observe that this reformulation implies the existence of local informative subspaces.
Finally, we show that when the system converges to a one dimensional space, the reformulation can be solved exactly.

A key requirement for this analysis is a notion of stability of the nullspace of the Jacobian of the flow $\varphi^\tau$.
For this lemma, sensitivities to normal perturbations in the neighborhood of the attracting manifold must decay monotonically.
This allows the construction of an $M-K$ dimensional vectorspace at each $x$, commonly called a vector bundle, for which the local sensitivity of the flow converges to zero for all vectors.

\begin{lemma}[Proof: See Lemma \ref{prf:lipschitz}]
    \label{prop:lipschitz}
    If the system in Equation \eqref{eq:ssdyn} converges to a $K$-dimensional smooth manifold $\mathcal{M}$ such that there exists some $\alpha > 0$ for which $v^\top\left(\left[df_x\right] + \left[df_x\right]^\top\right)v + \alpha \|v||^2 < 0$ for each $v \in T_p^\perp\mathcal{M}$ normal to the manifold at each $p \in \mathcal{M}$, then there exists an $M-K$ dimensional subspace $S(x_0) \subset \mathbb{R}^M$ for which $\left[d\varphi^\tau_{x_0}\right] u \rightarrow 0$ as $\tau \rightarrow \infty$ for each $u \in S(x_0)$.
\end{lemma}

We now introduce the main theorem in this work.
By recalling that when a probability distribution is reparameterized the CRLB undergoes a linear transformation based on the Jacobian of the reparameterization, previously shown in Equation \eqref{eq:reparam}, the information in measurements is restricted to that contained in a K-dimensional subspace.

\begin{theorem}[Proof: See Theorem \ref{prf:weakened}]
    \label{thm:weakened}
    Suppose the state space model in Equations \eqref{eq:ssdyn} and \eqref{eq:ssmodel} with noise satisfying the assumptions in Theorem \ref{thm:rank} converges to a $K$-dimensional smooth manifold such that the Jacobian of the limiting flow is rank $K$.
    Then the linear functional which minimizes the CRLB for prediction in the infinite-horizon average cost formulation is the solution of 
    \begin{equation}
        \argmax_{u: \|u\|_{(\Lambda_\sigma + \munderbar{\Sigma})} = 1} \sum_{i=1}^K \alpha_i \langle v_i, u \rangle_{\munderbar{\Sigma}}^2 \label{eq:weakformmain}
    \end{equation}
    for some $\{\alpha_i\}$, where ${v_i}$ is the set of right singular vectors of the Jacobian of the limiting flow around the current state with nonzero singular vectors, $\munderbar{\Sigma}$ is the CRLB for estimating the current state based on the past measurements, $\langle v, u \rangle_{\munderbar{\Sigma}} = v^\top \munderbar{\Sigma} u$, and $\|u_i\|_{(\Lambda_\sigma + \munderbar{\Sigma})} = \sqrt{u^\top (\Lambda_\sigma + \munderbar{\Sigma})u}$.
\end{theorem}

It is worth noting that this problem reformulation additionally holds for minimizing the trace of the covariance of biased estimators, as the bias introduces an additional linear transformation on the low-dimensional subspace \citep{lehmann_1998}.
The $\{\alpha_i\}$ coefficients become dependent on the bias of the estimator, but dimensionality reduction insights still hold.
Through Lagrange multipliers, we can compute the informative subspace exactly.

\begin{corollary}[Proof: See Corollary \ref{prf:basis}]
    \label{cor:basis}
    Under the assumptions of Theorem \ref{thm:weakened}, the optimal measurement vector exists in the subspace
    \begin{equation}
        \textup{Span}\,\left\{(\munderbar{\Sigma} + \Lambda_\sigma)^{-1} \munderbar{\Sigma} v_i \right\}_{i=1}^K.
    \end{equation}
\end{corollary}

We can consider this to be a local informative subspace, or an informative vector bundle.
By further observing that the information content is invariant to scaling of the measurement vector, we can construct an exact solution to the time-averaged infinite-horizon optimal experimental design problem.

\begin{corollary}[Proof: See Corollary \ref{prf:closed}]
    \label{cor:closed}
    If, in addition to the assumptions of Theorem \ref{thm:weakened}, the dynamical system converges to an isolated limit cycle, then 
    \begin{equation}
        u = (\munderbar{\Sigma} + \Lambda_\sigma)^{-1} \munderbar{\Sigma} v
    \end{equation}
    is an optimal measurement vector, where $v$ is the right singular vector.
\end{corollary}


At this point, it is worth considering practical issues which largely represent opportunities for future work.
For any given state, computing $\{v_i\}$ requires first-order local sensitivity analysis of the dynamical system.
As this is a well-studied problem with numerous existing techniques, we do not address this problem further in this manuscript \citep{BenchmarkTools.jl-2016, ma_comparision}.

The informative subspace is dependent on the state of the system. 
In Section \ref{sec:est}, we propose and discuss the usage of the state estimate produced by an extended Kalman filter due to the ability to reuse computation.

Furthermore, we have not addressed how to optimize Equation \eqref{eq:weakformmain} in the general case.
We believe this local dimensionality reduction may be useful to reduce scaling issues in general approaches, reducing the action space to the dimensionality of the limiting manifold.

\begin{figure*}[t]
    \centering
    \includegraphics[width=0.75\textwidth]{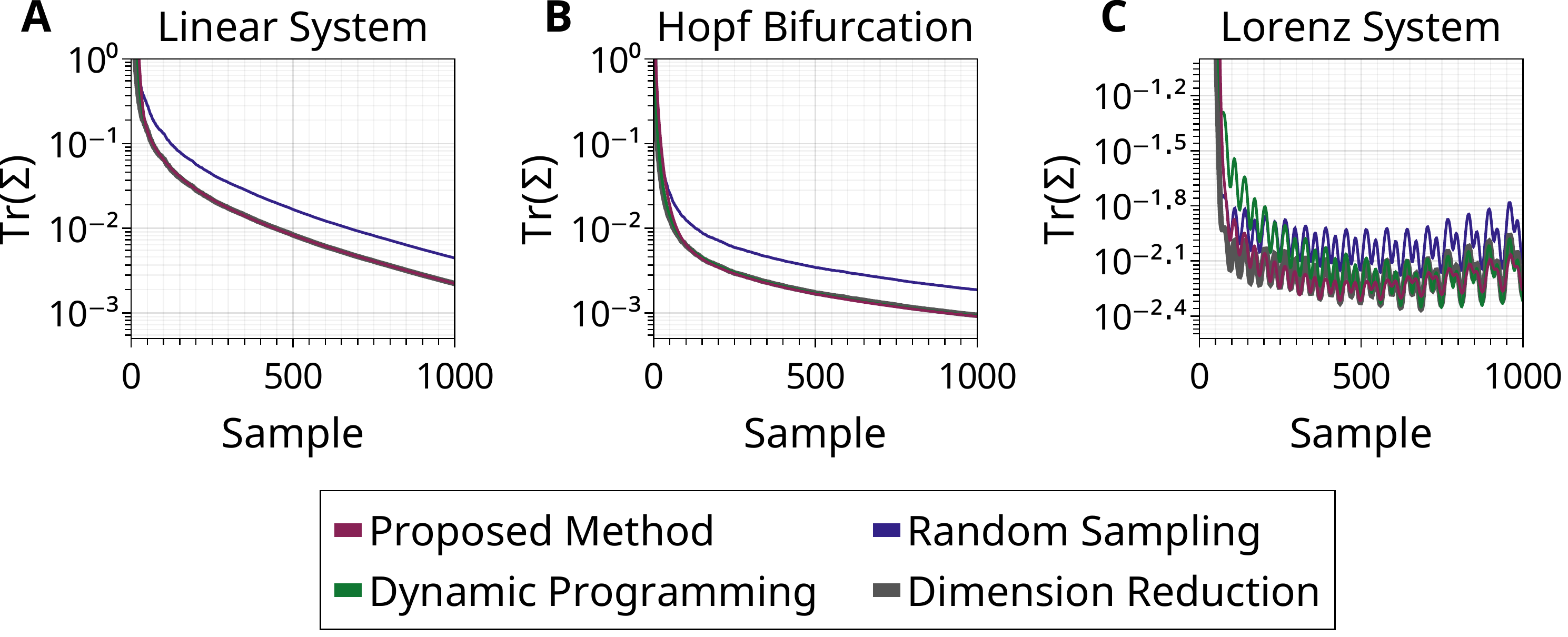}
    \caption{Using uniform random sampling as a baseline, the oracle solutions for both the computationally expensive dynamic programming method and our proposed method capture the dimensionality collapse on each of the three systems. A: Linear System; B: Hopf Bifurcation; C: Lorenz System}
    \label{fig:SystemSampling}
\end{figure*}

\section{STATE ESTIMATION}
\label{sec:est}
In order to practically apply either the Bellman equation based dynamic programming algorithm or the proposed algorithm, we must include some form of state estimation. 
Due to its close connection with the propagation of the CRLB, we illustrate our algorithms using the extended Kalman filter (EKF).

The EKF is a two-step state estimation procedure which involves a prediction step followed by an estimation step.
In its general form, the state vector is updated according to the system dynamics, then corrected based on the observation.
The covariance matrices are propagated through a linearization of state transition operator and measurement operator.
While the technique lacks the optimality guarantees of the original Kalman filter, it has remained a common method of state estimation in dynamical systems since its discovery in the 1960s \citep{ekf}.

In our application, the EKF is particularly efficient because the computationally expensive covariance update coincides with the CRLB update.
Furthermore, as our measurements are scalars, the typically required matrix inversion in the estimation step becomes scalar division.
When we incorporate our system model in Equations \eqref{eq:ssdyn} and \eqref{eq:ssmodel} and our noise model from Theorem \ref{thm:rank}, the EKF reduces to
\begin{align}
    \hat{x}_{i|i-1} &= \varphi^\tau\left(\hat{x}_{i-1|i-1}\right),\\
    \hat{x}_{i|i} &= \hat{x}_{i|i-1} +  \frac{\munderbar{\Sigma}_i u_i(y_i - u_i^\top \hat{x}_{i|i-1})}{u_i^\top (\Lambda_\sigma + \munderbar{\Sigma}_i)u},\label{eq:extra}\\
    \munderbar{\Sigma}_{i+1} &= \left[d\varphi_x^\tau\right] \left(\munderbar{\Sigma}_{i} - \frac{\munderbar{\Sigma}_{i}u_i u_i^\top \munderbar{\Sigma}_{i}}{u_i^\top (\Lambda_\sigma + \munderbar{\Sigma}_i)u}\right)\left[d\varphi_x^\tau\right]^\top,
\end{align}
where the covariance update is identical to the CRLB update, and can thus be reused.
The only additional computational cost to apply the EKF beyond the CRLB update is in Equation \eqref{eq:extra}, which is negligible compared to the already required costs.
A detailed description of operation counts is available in Appendix \ref{app:computation}.

\section{EXPERIMENTAL RESULTS}
The numerical simulations in this section were chosen to illustrate three key features of our proposed experimental design procedure.
We begin by illustrating that, when the state of the system is known exactly, both the proposed experimental design procedure and a computationally expensive dynamic programming baseline (Appendix \ref{app:approx}) achieve the expected CRLB reduction suggested by the dimensionality collapse.
We then demonstrate that the application of the EKF is sufficiently accurate to maintain this improvement after a short delay. 
Finally, we demonstrate that the benefit improves linearly with the ambient dimensionality of the system.
Simulations regarding computation time appear in Appendix \ref{app:computation} as further justification of the computational efficiency of the proposed algorithm.

Throughout this section, we simulate representative examples of three important classes of systems: linear systems, limit cycles, and chaotic systems.
As many real-world systems exhibit these fundamental behaviors in different operating regimes, our chosen examples help to elucidate how the proposed algorithm will perform in each behavior.
The three characteristic systems on which the observation policies were evaluated were a stable linear system defined with eigenvalues $\lambda_1 = -10$ and $\lambda_2 = -0.1$, a Hopf bifurcation with dynamics $\{\dot{x}_{(1)} = x_{(1)}(1 - \|x\|^2) - x_{(2)}; \dot{x}_(2) = x_{(2)}(1 - \|x\|^2) + x_{(1)}\}$, and a Lorenz system with dynamics
$\{ \dot{x}_{(1)} = 10(x_{(2)}-x_{(1)}); \;
    \dot{x}_{(2)} = x_{(1)}(28-x_{(3)}) - x_{(2)}; \;
    \dot{x}_{(3)} = x_{(1)}x_{(2)} - \frac{8}{3} x_{(3)}\}$, where $x_{(i)}$ indicates the $i$'th element of the state vector.
While chaotic systems do not meet our theoretical convergence assumption, we include the Lorenz system to characterize performance on this related behavior.
The specific parameters are available in Appendix \ref{app:details}.

\begin{figure*}
    \centering
    \includegraphics[width=0.85\textwidth]{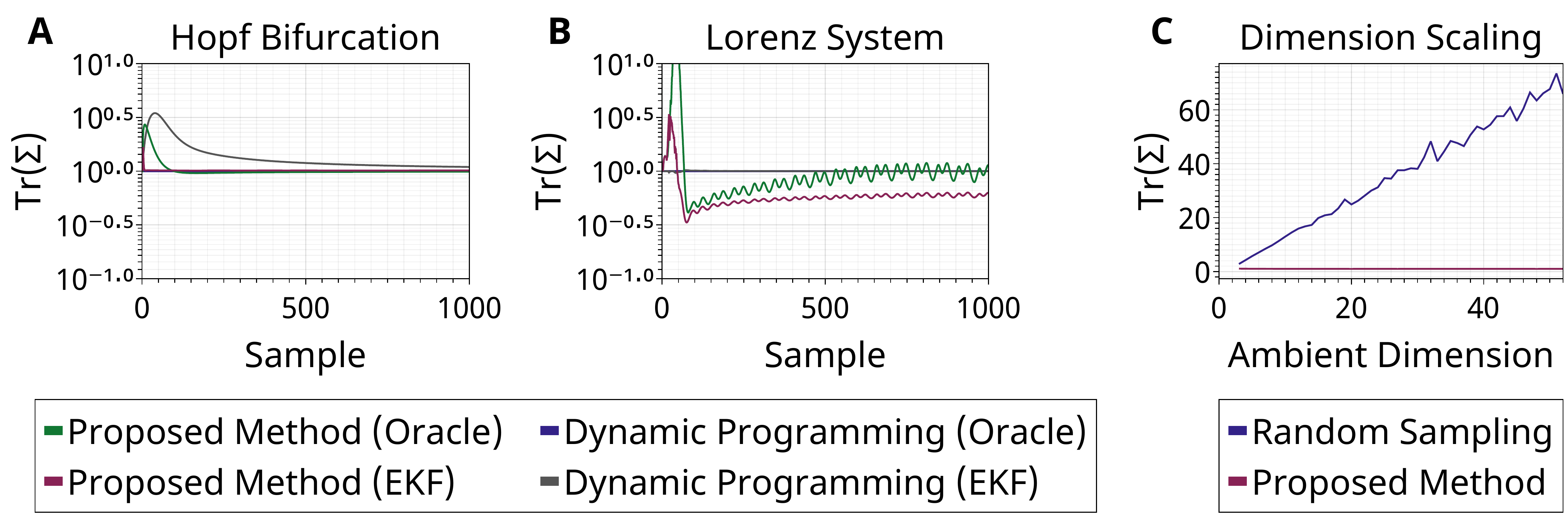}
    \caption{(A,B): Trace of the CRLB normalized to that of the dynamic programming oracle solution for the (A) Hopf Bifurcation system and (B) Chaotic Lorenz system; (C): Trace of the CRLB of estimating the state of a limit cycle system normalized to that of the proposed method as a function of ambient dimensionality. Note the linear growth of the baseline random sampling approach.}
    \label{fig:scaling}
\end{figure*}

In Figure \ref{fig:SystemSampling}, we illustrate that the oracle solution of the optimal sampling procedure works across representative examples of each class of systems in which the true state is known to the decision process.
A line corresponding the optimal CRLB reduction enabled by the loss of dimensionality is included for reference.
The key observation in this figure is that all techniques approximately capture the full improvement potential.
Furthermore, as the plot represents the CRLB for state estimation rather than forecasting, we observe the fundamental trade-off made in this work: while we attain an improvement for future time-horizons, our ability to estimate the current state is negatively impacted in the short term.
This is particularly visible in the dynamic programming plot in Panel C where the CRLB reduction does not materialize for about 250 samples, or 2.5 seconds of the system trajectory.

Next, we introduce the partially observed setting with the EKF in Panels A and B of Figure \ref{fig:scaling}, normalizing the trace of the CRLB to that of the oracle solution in the Dynamic Programming approximation.
Because the plots again illustrate the CRLB for the current state, we observe a short-term performance benefit for the EKF approach, as it is less aggressive in prioritizing the future.
As the Lorenz system is not truly 1D, we see a similar benefit in Panel B.
In these simulations, we observe empirically that the EKF based policy often converges to similar levels of performance to the oracle policy, particularly in the proposed method.

Finally, in Panel C of Figure \ref{fig:scaling}, we introduce a system based on expanding a Van der Pol oscillator, a common example of a system with a limit cycle, with additional dimensions governed by stable linear dynamics to illustrate the linear improvement as a function of the number of collapsing dimensions.
In this system, the dynamics of the first two state variables are defined as $\{\dot{x}_{(1)} = 3.5(x_{(1)} - \frac{1}{3}x_{(1)}^4 - x_{(2)});\quad \dot{x}_{(2)} = \frac{2}{7}x_{(1)}\}$, while the remaining dimensions all follow $\dot{x}_{(i)} = -x_{(i)}$.
Through these simulations, we remove any confounding issues and focus purely on the scaling with the dimensionality.
The blue curve represents a baseline approach in which, at each decision point, the measurement vector is chosen uniformly at random over all unit vectors.
This baseline represents an efficient method to collect measurements about the current state, particularly in the case when observations are made quickly relative to the rate of change of state in the system.
We see that, in this greedy baseline approach, the trace of the future CRLB increases linearly with dimensionality of the system while our proposed technique preserves a nearly constant value.

\section{CONCLUSION}
\label{sec:conclusion}
In this work, we introduced a dimensionality reduction method for the sequential selection of linear functional measurements of a vector-valued time series when the system converges to a low-dimensional manifold.
Based on theoretical properties of natural exponential families of probability distributions, we reformulated the average cost infinite horizon problem for systems which converge to low-dimensional sets, resulting in a computationally efficient alternative to a Bellman equation approach.

Beyond theoretical results, we demonstrated the performance on three different common classes of dynamical systems: linear systems, limit cycles, and chaotic systems.
Simulations were completed illustrating the performance of the dynamic programming formulation and the 1D reformulation using an extended Kalman filter for state estimation.
Simulations showed that the measurement selection performance remained constant irrespective of the ambient dimensionality of the space --- in contrast with i.i.d. uniform random measurements which grow linearly.


The contributions in this work enable the collection of significantly more informative data than would be acquired through uniform observations.
This improvement is attained irrespective of the eventual forecasting algorithm.
We believe this work opens new directions in forecasting time series, and anticipate follow-on work regarding the co-design of observation and forecasting algorithms.

\subsubsection*{Acknowledgements}
This work was supported in part by the National Science Foundation under Grant 1936663.

\nocite{chaloner_1995,Dette_designing}
\nocite{martinez_2006,ranieri_2014,xygkis_2018}
\nocite{fohring_adaptive,asprey_designing,Huan_sequential,titterington_aspects}
\nocite{letham_2016,Transtrum_2012,casey_optimal}
\nocite{penland_prediction_1993,Ehrendorfer_1994,ehrendorfer_2006}
\nocite{Abu-Mostafa_1996,Cao_2001}
\nocite{Bertsekas_2017_1}
\nocite{shi_2017,Ravuri_2021}
\nocite{Gilphi_2020}
\nocite{Brunton_2016,Kaiser_2018,Zhang_2019,messenger_2021}
\nocite{schmid_2010,williams_2015,proctor_2016,erichson_2019}
\nocite{raissi_2019}
\nocite{raissi_2018}
\nocite{chen_2018,dupont_2019,dandekar_2020,massaroli_2020,avelin_2021}
\nocite{greydanus_2019,cranmer_2020}
\nocite{bai_2019,bai_2020,pabbaraju_2021,gilton_2021}
\nocite{khalil_2002}
\nocite{khalil_2002}
\nocite{khalil_2002}
\nocite{hajek_2010}
\nocite{Morris_2006}
\nocite{Bertsekas_2017}
\nocite{mezzadri_2006}
\nocite{mezzadri_2006}
\nocite{sherman}
\nocite{trefethen_1997}
\nocite{sherman}
\nocite{rackauckas2017differentialequations}
\nocite{ma_comparision}
\nocite{RevelsLubinPapamarkou2016}
\nocite{BenchmarkTools.jl-2016}
\nocite{DanischKrumbiegel2021}

\bibliography{refs}
\bibliographystyle{plainnat}

\appendix
\onecolumn

\section*{APPENDIX}

The appendix is organized as follows.
Appendix \ref{app:table} describes the systems shown in Table \ref{tab:examples}.
Appendix \ref{app:approx} provides derivations of the performance of the dynamic programming base case.
Appendix \ref{app:oracle} provides the proofs for Section \ref{sec:oracle}.
Appendix \ref{app:computation} contains counts of multiplications required in the algorithms, as well as timing simulations and details regarding the gradient ascent procedure.
Appendix \ref{app:numerical} discusses potential numerical issues in the work relating to the requirement to invert approximately singular matrices, as well as the chosen solution to the problem.
Appendix \ref{app:code} describes the interface for incorporating new dynamical systems into the library produced in the course of this work. 
Finally, Appendix \ref{app:details} includes all parameter used in generating the plots in the main body of the manuscript.

\section{TABLE \ref{tab:examples} DETAILS}
\label{app:table}
Representative trajectories for the systems in Table \ref{tab:examples} are shown in Figure \ref{fig:examples}.
Initializations were chosen in an arbitrary, often randomized manner as the chosen systems are known to contain only a small number of stable global behaviors
In particular, all systems other than the gradient flow globally converge to a fixed behavior, while the chosen gradient flow contains only two stable equalibria.

\begin{figure}[ht]
    \centering
    \includegraphics[width=\textwidth]{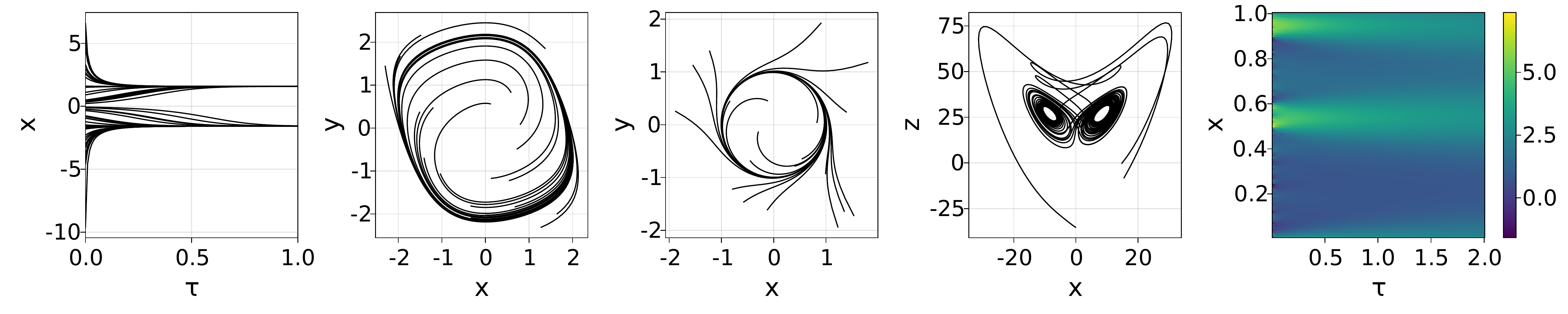}
    \caption{From left to right, Morse Flow, Van Der Pol, Hopf, Lorenz, Heat equation}
    \label{fig:examples}
\end{figure}

\fakesection{Gradient Flow}

Let $h: \mathbb{R}^M \rightarrow \mathbb{R}$ on some bounded interval be a Morse function.
The gradient flow is defined as
\begin{equation}
    \dot{x} = -\nabla h(x).
\end{equation}
The key property of a Morse function in this example is that all critical points occur in the interior of the manifold $\mathcal{M}/\partial\mathcal{M}$ and that no critical points are degenerate.
The lack of degeneracy is well-known to imply that critical points are isolated, thus resulting in a discrete-set of points.

In the particular simulation for Figure \ref{fig:examples}, $h(x) = x^4 - 5x^2 + 4$.

\fakesection{Van Der Pol Oscillator}

\begin{align}
    \dot{x} &= y\\
    \dot{y} &= \mu(1-x^2)y - x
\end{align}

The Van Der Pol Oscillator is well-known to have a single globally stable limit cycle, though there is no closed form.

In the Figure \ref{fig:examples}, $\mu = 1$.

\fakesection{Hopf Bifurcation}

\begin{align}
    \dot{x} &= x(\lambda + b (x^2 + y^2)) - y\\
    \dot{y} &= y(\lambda + b (x^2 + y^2)) + x
\end{align}

The parameters are chosen such that the system exhibits a stable limit cycle, for example $b = -1$ and $\lambda = 1$ in Figure \ref{fig:examples}.
Under the above choice of parameters, there is a stable limit cycle at $x^2 + y^2 = 1$.
By substituting in the constraint into the above expression, the system becomes
\begin{align}
    \dot{x} &= -y \\
    \dot{y} &= x
\end{align}
which is linear with purely imaginary eigenvalues.

\fakesection{Lorenz System}

\begin{align}
    \dot{x} &= \sigma(y-x) \\
    \dot{y} &= x(\rho-z) - y\\
    \dot{z} &= xy - \beta z
\end{align}

The almost canonical set of parameters for the system is $\sigma = 10$, $\rho = 28$, $\beta = 8/3$, which generates the butterfly-shaped attractor pictured in Figure \ref{fig:examples}.
In the case of this well-known chaotic dynamical system, the low-dimensional structure is less obvious.
In fact, the attractor is not embedded in a 2D manifold at all.
Instead, due to deep connections between chaos and fractals, chaotic attractors are often described by their fractal dimension.
While there are many definitions of fractal dimensions, a common definition is the Hausdorff dimension, defined based on covering the space with arbitrarily small balls.
In the case of the Lorenz system with the chosen parameters, the dimensionality is well-known to be approximately $2.06$.

\fakesection{Heat Equation}

The heat equation is one of the most fundamental linear partial differential equations, and represents spatial diffusion over time.

\begin{equation}
    \frac{\partial u}{\partial t} = \alpha \Delta u
\end{equation}

In Figure \ref{fig:examples}, the heat equation is simulated according to a finite-difference discretization with $\alpha=0.2$.
The only eigenfunction of the system which does not exponentially decay is the average value of $u$, and so the system converges only to a 1D space.

\section{CRAM\'{E}R--RAO LOWER BOUND}
\label{app:crlb}
In this appendix, we provide a brief overview of the well-known Cram\'{e}r--Rao Lower Bound (CRLB), with regularity conditions adapted from \cite{lehmann_1998}.

The multivariate CRLB, sometimes known as an information inequality, is a matrix inequality on the covariance matrix of all estimators of an unknown, deterministic parameter of a probability distribution.
Specifically, let $p(\xi; \theta)$ be a family of density functions of a random variable $\xi$ parameterized by a vector $\theta$ such that $p$ satisfies the regularity conditions:
\begin{enumerate}
    \item $\Theta$ is an open set of possible parameters $\theta$.
    \item The densities $p(\cdot; \theta)$ have common support. 
    \item For each $\xi \in \Xi, \theta \in \Theta$, and $i \in 1, \ldots, M,$ $\frac{\partial p(\xi; \theta)}{\partial \theta_i} < \infty$.
\end{enumerate}

Then, for any estimator of the parameter $\hat{\theta}$ such that $i \in 1, \ldots, M,$ $\frac{\partial}{\partial \theta_i} \mathbb{E}\left[\hat{\theta}\right] = \mathbb{E}\left[\frac{\partial}{\partial \theta_i} \hat{\theta}\right]$, 
\begin{equation}
    \Sigma \geq \left[\frac{d\mathbb{E}\left[\hat{\theta}\right]}{d\theta}\right]^\top J^{-1}(\theta) \left[\frac{d\mathbb{E}\left[\hat{\theta}\right]}{d\theta}\right],
\end{equation}
where $\Sigma$ is the covariance matrix of the estimator and $J^{-1}(\theta)$ is the inverse Fisher information matrix.
The matrix inequality is interpreted such that a matrix $A \geq 0$ is positive semi-definite.
When applied to unbiased estimators where $\mathbb{E}\left[\hat{\theta}\right] = \theta$, the CRLB reduces to the exceedingly common unbiased form
\begin{equation}
    \Sigma \geq  J^{-1}(\theta). 
\end{equation}

The Fisher information matrix is the covariance of the derivative of the log-likelihood function, i.e.
\begin{equation}
    J(\theta) = \mathbb{E}\left[\left(\frac{d}{d\theta}\log p(\xi; \theta)\right)\left(\frac{d}{d\theta}\log p(\xi; \theta)\right)^\top\right].\label{eq:fim}
\end{equation}

Of note in this work is the impact of diffeomorphic reparameterization on the Fisher information matrix, which can be seen immediately by applying the chain rule to Equation \eqref{eq:fim}.

\section{DYNAMIC PROGRAMMING BASELINE}
\label{app:approx}
While the dynamic programming problem maps directly into the stochastic control and reinforcement learning frameworks, there are still a number of application-specific details in all such algorithms.
In this section, we address the choice of an appropriate discretization scheme for the state space.

First, we need to consider the structure of the state space and the action space.
The policy is dependent both on the previous CRLB and on the previous system state.
But note that the CRLB is a positive semi-definite (PSD) matrix, and so if the system state is in $\mathbb{R}^M$, then the space of CRLB matrices is embedded in $\mathbb{R}^{(M^2 + M)/2}$.
Furthermore, the space of PSD matrices trivially forms a cone, which further reduces the space.
An equivalent representation of the space is through the diagonalization $A = U \Lambda U^\top$, where $\Lambda$ is a diagonal matrix and $U \in O(M)$ is an orthonormal matrix.

The space of PSD matrices has a non-trivial structure, yet we have strong heuristics for the problem.
Thus, we opt to approximate through the interpolation of random sample points.
We denote the set of sample points by $\mathcal{X} = \{x_1,...,x_N\}$ and our approximation of the value function to be $\hat{J}_{\mathcal{X}}(x, \munderbar{\Sigma})$.
We first describe the sampling distribution of $\mathcal{X}$, before discussing the interpolation methods.

\fakesection{Random Sampling}
A practical approach to generating random PSD matrices is to decompose the problem into sampling from the space of orthonormal matrices and sampling non-negative eigenvalues.
Not only is this approach straightforward, but as we have more intuition regarding the distribution of eigenvalues, it allows easy tuning of the approximation of the distribution.
The full sampling routine is shown in Algorithm \ref{alg:sampling}.

\begin{algorithm}[tb]
    \caption{Interpolation Point Sampling}
    \label{alg:sampling}
 \begin{algorithmic}
    \STATE {\bfseries Input:} Eigenvalue Distribution $\mu$
    \STATE Sample diagonal matrix $\Lambda \sim \mu$ \COMMENT{$\left(\Lambda \in \mathbb{R}^{M \times M}\right)$}
    \STATE Sample $G\sim \mathcal{N}(0,I)$ \COMMENT{$\left(G \in \mathbb{R}^{M \times M}\right)$}
     \STATE $Q, R \gets QR(G)$ \COMMENT{$\,\left(Q \sim\text{ Haar Measure}\right)$}
    \STATE $x_i \gets Q^\top \Lambda Q$ \COMMENT{$\left(x_i \in \mathbb{R}^{M \times M}\right)$}
 \end{algorithmic}
 \end{algorithm}

Using a naive prior to model the problem, we begin with sampling the orthonormal matrix components from a Haar measure.
The Haar measure can be thought of as a uniform distribution --- it is the unique measure that is invariant to the action of elements in the group \citep{mezzadri_2006}.
It is exceedingly easy to generate random orthonormal matrices from this measure.
It can be shown that due to the isotropic behavior, diagonalizing vectors of a matrix populated with i.i.d. Gaussian random variables results in such a distribution.
Practically, sampling is done using QR factorization, typically with an eigenvalue correction term that cancels in similarity transforms \citep{mezzadri_2006}.

Understanding that each measurement is optimized to minimize the eigenvalues of the CRLB, we follow the heuristic that smaller eigenvalues will be encountered more frequently and sample from an exponential distribution.
The parameter of the distribution is then a hyperparameter of the problem.

\fakesection{Interpolation Methods}
While one could choose any one of a large number of interpolation methods, in this work, we only analyze nearest neighbor and a locally averaged perturbation of nearest neighbor.
While alternative approaches may yield higher accuracy, the approximation quality for classification tasks is extremely well studied with strong guarantees, typically providing both upper and lower bounds due to the connection to density estimation techniques.

It is clear that the CRLB of our measurement system will be unevenly distributed in any practical setting: as we gather more data, the CRLB decreases.
Thus, we would like to weight our approximation quality to account for this phenomenon.
Here, we provide a bound for function approximation accuracy under a non-uniform measure defined by a probability distribution.

\begin{lemma}
    Assume for every $x \in \mathcal{X}$, the approximation of the value function is exact, i.e. $\hat{J}_\mathcal{X}(x) = J(x)$, where $\hat{J}_{\mathcal{X}}$ is our approximation.
    Furthermore, assume $J$ is Lipschitz with coefficient $L$.
    Then the mean absolute error of the function approximation is upper bounded by
    \begin{equation}
        \mathbb{E}_{\mathcal{X},x}\left[|\hat{J}_\mathcal{X}(x) - J(x)|\right] \leq L \mathbb{E}_{\mathcal{X}, x}\left[\min_{x' \in \mathcal{X}} \|x - x'\|\right]
    \end{equation}
    \label{lemma:lip}
\end{lemma}
\begin{proof}
    By the Lipschitz assumption, $|J(x) - J(\tilde x)| \leq L\|x - \tilde x\|$, and so
    \begin{equation}
        \mathbb{E}_{\mathcal{X},x}\left[|J(x) - J(\tilde x)|\right] \leq L \mathbb{E}_{\mathcal{X}, x}\left[\|x - \tilde x\|\right].
    \end{equation}
    By construction of nearest neighbor regression, $\hat{J}_{\mathcal{X}}(x) = J(x')$ where $x' = \argmin_{\tilde x \in \mathcal{X}}\|x-\tilde x\|$.
    Thus
    \begin{equation}
        \mathbb{E}_{\mathcal{X},x}\left[|\hat{J}_\mathcal{X}(x) - J(x)|\right]  = \mathbb{E}_{\mathcal{X},x}\left[|J(x) - J(x')|\right] \leq L \mathbb{E}_{\mathcal{X}, x}\left[\min_{x' \in \mathcal{X}} \|x - x'\|\right]
    \end{equation}
\end{proof}

The expected minimum distance to the dataset is dependent on the distribution and is non-trivial to compute analytically.
We include a simulation in Figure \ref{fig:distances} illustrating the expected minimum distance as a function of the number of samples.
One important observation is the well-known curse of dimensionality, and so this particular function approximation would scale poorly to higher dimensions.

\begin{figure}
    \centering
    \includegraphics[width=0.47\textwidth]{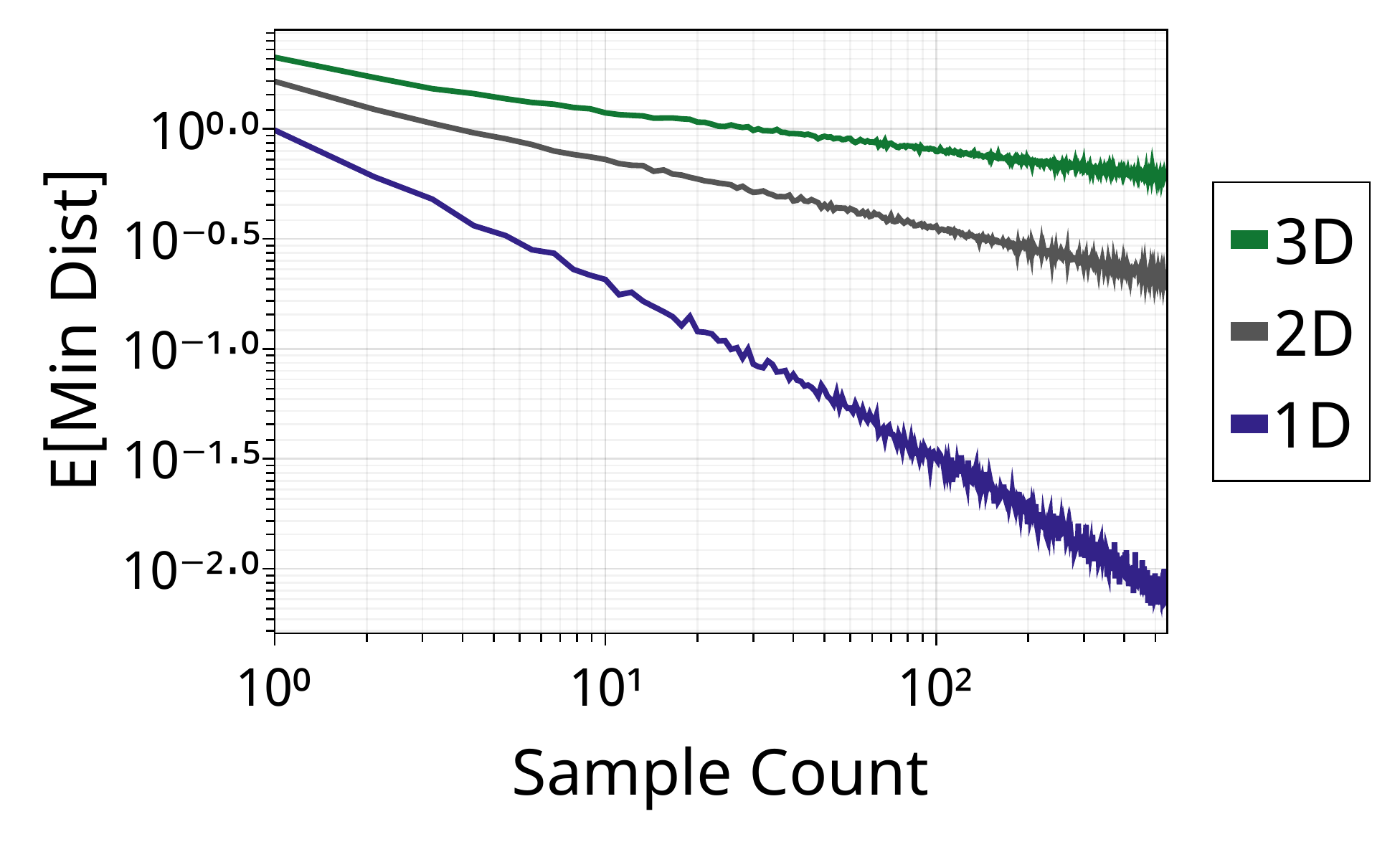}
    \caption{Numerically estimated expected minimum distance as a function of the size of the dataset.}
   \label{fig:distances} 
\end{figure}

Finally, the nearest neighbor model fails to produce reasonable decisions when the CRLB changes by less than twice the distance between sample points --- resulting in the state staying constant in value iteration steps.
To prevent this issue, we apply local averaging motivated by linear interpolation.
In particular, in our local averaging model, an approximate point is evaluated as
\begin{equation}
    \hat{J}_\mathcal{X}^{(loc)}(\hat{x}) = \frac{\sum_{x \in \mathcal{X}} J(x) \min\left(0,d_{max} - \|x-\hat{x}\|\right)}{\sum_{x \in \mathcal{X}} \min\left(0, d_{max} - \|x - \hat{x}\|\right)},
\end{equation}
where $d_{max}$ is a chosen maximum distance for averaging and $\hat{J}_\mathcal{X}(\hat{x})$ is replaced by a nearest neighbor approximation if the denominator becomes $0$.
We can readily bound the perturbation of our locally averaged interpolation from our nearest neighbor interpolation based on the Lipschitz assumption.
\begin{theorem}
    The mean absolute error of the locally averaged approximation is at most
    \begin{equation}
\mathbb{E}_{\mathcal{X},x}\left[|\hat{J}_\mathcal{X}^{(loc)}(x) - J(x)|\right] \leq L\left(\mathbb{E}_{\mathcal{X}, x}\left[\min_{x' \in \mathcal{X}} \|x - x'\|\right]+ 2 d_{max}\right).
    \end{equation}
\end{theorem}
\begin{proof}
    The inequality comes from observing the maximum perturbation from the nearest neighbor solution.
    
    First, observe that if the $d_{max}$ ball around $x$ contains $1$ or fewer points, then there is no change to the approximation.
    Next, consider the case when there are exactly two representation points in the ball.
    Then $\hat{J}_{\mathcal{X}}^{(loc)}(x) = \frac{(d_{max} - d_1)J(x_1) + (d_{max}-d_2) J(x_2)}{2d_{max} - d_1 - d_2}$, where $d_i$ is the distance from $x$ to $x_i$.
    Now suppose $d_1 < d_2$, then $\hat{J}_{\mathcal{X}}(x) = J(x_1)$.
    \begin{equation}
        \hat{J}_\mathcal{X}(x) - \hat{J}_{\mathcal{X}}^{(loc)}(x) = (J(x_1) - J(x_2))\frac{d_{max} - d_2}{(d_{max} - d_1) + (d_{max} - d_2)}
    \end{equation}
    Now, $d_1 < d_2 \Rightarrow (d_{max} - d_1) > (d_{max} - d_2) \Rightarrow \frac{d_{max} - d_2}{(d_{max} - d_1) + (d_{max} - d_2)} \leq \frac{1}{2}$.
    It then follows that
    \begin{equation}
        \left|\hat{J}_\mathcal{X}(x) - \hat{J}_{\mathcal{X}}^{(loc)}(x)\right| \leq \frac{1}{2} \left|J(x_1) - J(x_2)\right|
    \end{equation}
    The distance between $x_1$ and $x_2$ is upper bounded by $2d_{max}$ which represents opposite ends of the ball.
    Thus, by the Lipschitz assumption
    \begin{equation}
        \left|\hat{J}_\mathcal{X}(x) - \hat{J}_{\mathcal{X}}^{(loc)}(x)\right| \leq L d_{max}.
    \end{equation}

    Consider the implication of adding additional points to the ball.
    Note that again, the worst-case scenario is when all points are nearly equidistant from the center, resulting in 
    \begin{equation}
        \left|\hat{J}_\mathcal{X}(x) - \hat{J}_{\mathcal{X}}^{(loc)}(x)\right| \leq \frac{1}{N} \left|J(x_1) - \sum_{i=2}^NJ(x_i)\right|.
    \end{equation}
    Once again, $J(x_1) - 2Ld_{max} \leq J(x_i) \leq J(x_1) + 2 Ld_{max}$, and thus
    \begin{equation}
        \left|\hat{J}_\mathcal{X}(x) - \hat{J}_{\mathcal{X}}^{(loc)}(x)\right| \leq \left(\frac{N-1}{N}\right)2 L d_{max} 
    \end{equation}
    By observing $(N-1)/N < 1$ for all $N>1$, we find one irrespective of the number of points in the ball.
    \begin{equation}
        \left|\hat{J}_\mathcal{X}(x) - \hat{J}_{\mathcal{X}}^{(loc)}(x)\right| \leq 2 L d_{max} 
    \end{equation}
    Finally, apply the triangle inequality to conclude the original statement, i.e.
    \begin{align}
        \left|\hat{J}_{\mathcal{X}}^{(loc)} - J(x)\right| &= \left|\hat{J}_{\mathcal{X}}^{(loc)} - \hat{J}_{\mathcal{X}}(x) + \hat{J}_{\mathcal{X}} - J(x)\right| \\
        &\leq \left|\hat{J}_{\mathcal{X}}^{(loc)} - \hat{J}_{\mathcal{X}}(x)\right| + \left|\hat{J}_{\mathcal{X}} - J(x)\right| \\
        &\leq L \left(2 d_{max} + \min_{x' \in \mathcal{X}}\|x - x'\|\right)
    \end{align}
    \begin{equation}
        \Rightarrow \mathbb{E}_{\mathcal{X},x}\left[|\hat{J}_\mathcal{X}^{(loc)}(x) - J(x)|\right] \leq L\left(\mathbb{E}_{\mathcal{X}, x}\left[\min_{x' \in \mathcal{X}} \|x - x'\|\right] + 2d_{max}\right).
    \end{equation}
\end{proof}
Conveniently, as $N$ increases, we can decrease $d_{max}$ while retaining the averaging behavior.
Thus, we choose $d_{max}$ proportional to expected minimum distance to preserve the behavior of the bound.

\section{PROOFS FOR SECTION \ref{sec:oracle}}
\label{app:oracle}

\begin{lemma}[Context: See Lemma \ref{lemma:nef}]
    \label{prf:nef}
    Given a finite set of independent random variables $\{\xi_{(i)}\}$ in NEFs with log-partition functions $\{A_i\}$, parameters $\{\theta_i\}$, and a set of real-valued weights $\{\alpha_i\}$, then $\bar \xi = \sum_i \alpha_i \xi_{(i)}$ is in a NEF with a log partition function
    \begin{equation}
        A(\theta) := \sum_i A_i\left(\theta +  \alpha_i\theta_i\right).
    \end{equation}
\end{lemma}
\begin{proof}
Without loss of generality, we prove the case of two random variables. 
The result can then be trivially extended to any finite set, either by induction or through a modification of the below proof.

The moment generating function of random variable $X$ in a NEF is well-known to be
\begin{equation}
    M_X(s) = \exp\{A(\theta + s) - A(\theta)\}.
\end{equation}
If we denote the log-partition function of $\xi_i$ to be $A_i$, then
\begin{equation}
    M_{\xi_1 + \xi_2}(s) = \exp\{A_1(\theta_1 + s) + A_2(\theta_2 + s) - A_1(\theta_1) - A_2(\theta_2)\}.
\end{equation}
We then define the new log-partition function $A: s \mapsto A_1(\theta_1 + s) + A_2(\theta_2 + s)$.
Thus, by the structure of the MGF, the distribution itself is in a NEF.
\end{proof}

\begin{lemma}[Context: See Lemma \ref{lemma:fisher}]
    \label{prf:fisher}
The Fisher information of a random variable in a NEF is the variance of the random variable.
\end{lemma}
\begin{proof}
    By definition of the NEF, the log-likelihood takes the form
    \begin{equation}
        \theta \xi - A(\theta) + B(\xi).
    \end{equation}
    Recall that $A'(\theta) = \mathbb{E}[\xi]$.
    Thus, by computing the derivative, the Fisher information is 
    \begin{equation}
        \mathbb{E}[(\xi - A'(\theta))^2] = \mathbb{E}[(\xi - \mathbb{E}[x])^2]
    \end{equation}
\end{proof}

\begin{lemma}[Context: See Lemma \ref{lemma:diagram}]
    \label{prf:diagram}
    The diagram in Figure \ref{fig:commute} commutes for any vector $\xi$ of independent random variables from NEFs.
\end{lemma}
\begin{proof}
    We must show that $\{g_i\} \circ h$ = $h \circ g$.
    It suffices to show that each connection as drawn is accurate.
    The application of $g$ and $\{g_i\}$ are well-known from the NEF.
    For completeness, recall that the expected value of the score function of a random variable is $0$, i.e. 
    \begin{equation}
        \mathbb{E}[\frac{\partial}{\partial \theta} \log p(x;\theta)|\theta] = \frac{\partial}{\partial \theta}\mathbb{E}[\log p(x;\theta)|\theta] = 0.
    \end{equation}
    Then
    \begin{align}
        \mathbb{E}[\frac{\partial}{\partial \theta}(x\theta - A(\theta) + B(x))] &= \mathbb{E}[x - \frac{\partial}{\partial \theta}A(\theta)] \\
        &= \mu - g(\theta) = 0 \\
        &\Rightarrow \mu = g(\theta)
    \end{align}

    Furthermore, $h: \{\mu_i\} \mapsto \mu$ follows by linearity of expectation.

    Thus, the only tricky connection is to show $\theta = h(\{\theta_i\})$.
    Fortunately, this follows by observing we have flexibility in the definition of the NEF of $\bar{\xi}$.
    In particular, adding a constant offset to our definition of $\theta$ does not impact the definition of the family.
    Thus, we are free to center $\theta$ and our NEF definition around $\theta = h(\{\theta_i\})$.
\end{proof}

\begin{lemma}[Context: See Lemma \ref{lemma:inverse}]
    \label{prf:inverse}
    Assume the log-partition function $A$ of a NEF is twice differentiable and that the second moment of the associated random variable exists.
    Then the inverse of $g(\theta) := \left.\frac{\partial}{\partial \Theta} A(\Theta) \right|_{\Theta = \theta}$ exists if and only if the random variable is non-degenerate for almost all $\theta$.
\end{lemma}
\begin{proof}
    First, assume that $\xi$ is non-degenerate.
    Then, we will show $g$ is strictly monotonic and thus invertible.
    Recall that $g := A'$, where $A'$ is the derivative of $A$ with respect to $\theta$.
    Furthermore, recall that under the above regularity conditions, $A''$ is the Fisher information.
    Finally, by Lemma \ref{lemma:fisher}.
    \begin{equation}
        \frac{\partial}{\partial \theta}g(\theta) = A''(\theta) = \mathbb{E}[(\xi - \mathbb{E}[\xi])^2] = \text{Var}(\xi) > 0.
    \end{equation}
    Thus, $g$ is strictly increasing and is therefore invertible.
    
    Now, by contrapositive, assume there exists an interval for which $\xi$ deterministically takes a constant value.
    Then by the above argument, $\frac{\partial}{\partial \theta}g(\theta) = 0$ for all $\theta$ on the interval, and thus $g$ is not invertible on the interval.
\end{proof}

\begin{theorem}[Context: See Theorem \ref{thm:rank}]
    \label{prf:rank}
    Suppose $\xi$ is a vector of independent random variables from NEFs with twice differentiable log-partition functions, mean vector $\mu$, and a diagonal covariance matrix $\Lambda_\sigma$.
    Then the Fisher information of the observation $\langle u, \xi\rangle$ is 
    \begin{equation}
        I(\mu) = \frac{u u^\top}{u^\top \Lambda_\sigma u}
    \end{equation}
\end{theorem}
\begin{proof}
    The proof follows from the transformation to reparameterized Fisher information.
    That is, we can instead write $\bar \theta = T(\mu)$.
    \begin{align}
        I(\mu) &= \left[dT_{\mu}\right] I(\bar \theta) \left[dT_{\mu}\right]^\top \\
        &= \left[dh_{\mu}\right]dg^{-1}_{\bar \mu} I(\bar \theta) dg^{-1}_{\bar \mu} \left[dh_{\mu}\right]^\top
    \end{align}
    Now, recall that $I(\bar \theta) = dg_{\bar \mu} = \text{Var}(\bar \xi) = \sum_i u_{(i)}^2 \sigma_i^2 = u^\top \Lambda_\sigma u$, where $u_{(i)}$ represents the $i$'th element of $u$.
    Then by noting the inverse commutes with the derivative and that scalar multiplication commutes,
    \begin{equation}
        I(\mu) = \left[dh_{\mu}\right]\left[dh_{\mu}\right]^\top \frac{1}{u^\top \Lambda_\sigma u}.
    \end{equation}
    Finally, observe $\left[dh_{\mu}\right] = u$ to conclude the proof.
\end{proof}

\begin{corollary}[Context: See Corollary \ref{cor:invariant}]
    \label{prf:invariant}
    The Fisher information for state estimation under i.i.d. noise is invariant to the scaling of the linear functional.
\end{corollary}
\begin{proof}
    Follows immediately from the previous corollary by writing it as 
    \begin{equation}
        I(\mu) = \sigma^{-2} \frac{u u^\top}{\|u_i\|^2} = \sigma^{-2}\left(\frac{u}{\|u\|}\right) \left(\frac{u}{\|u\|}\right)^\top,
    \end{equation}
    and observing the absolute homogeneity of norms.
\end{proof}

\begin{corollary}[Context: See Corollary \ref{cor:iid}]
    \label{prf:iid}
    Suppose the elements of the noise vector are i.i.d. with variance $\sigma^{2}$.
    Then
    \begin{equation}
        I(\mu) = \sigma^{-2} \frac{u u^\top}{\|u\|^2}
    \end{equation}
\end{corollary}
\begin{proof}
    As the noise is identically distributed, it has the same variance and the covariance matrix becomes a scaled version of the identity.
    From Theorem \ref{thm:rank}
    \begin{equation}
        I(\mu) = \frac{u u^\top}{u^\top \Lambda_\sigma u} = \sigma^{-2} \frac{u u^\top}{u^\top u} = \sigma^{-2}\frac{u u^\top}{\|u\|^2}
    \end{equation}
\end{proof}

\begin{lemma}[Context: See Lemma \ref{prop:lipschitz}]
    \label{prf:lipschitz}
    If the system in Equation \eqref{eq:ssdyn} converges to a $K$-dimensional smooth manifold $\mathcal{M}$ such that there exists some $\alpha > 0$ for which $v^\top\left(\left[df_x\right] + \left[df_x\right]^\top\right)v + \alpha \|v||^2 < 0$ for each $v \in T_p^\perp\mathcal{M}$ normal to the manifold at each $p \in \mathcal{M}$, then there exists an $M-K$ dimensional subspace $S(x_0) \subset \mathbb{R}^M$ for which $\left[d\varphi^\tau_{x_0}\right] u \rightarrow 0$ as $\tau \rightarrow \infty$ for each $u \in S(x_0)$.
\end{lemma}
Note that this proof is long.
As such, we have split some of the significant intermediate steps into italicised statements.
Proofs of these intermediate statements are terminated with $\blacksquare$, while the full proof is terminated with $\square$.

\begin{proof}
    By considering the sensitivity equation \citep{khalil_2002}
    \begin{align}
        \dot{x} &= f(x) \\
        \dot{\left[\frac{d\varphi^\tau}{dx}\right]} &= \left[\frac{df}{dx}\right] \left[\frac{d\varphi^\tau}{dx}\right],
    \end{align}
    we will see that the normal sensitivity is asymptotically stable once sufficiently close to the manifold.
    By our assumption that the trajectory converges, we can represent the state using Fermi coordinates on a tubular neighborhood of the manifold.
    That is, we can represent the space as the product of the manifold $\mathcal{M}$ and an $M-K$ dimensional vectorspace representing the normals of the manifold.
    If $\mathcal{M}$ is smooth, such a neighborhood always exists by the tubular neighborhood theorem \citep{lee_2018}.

    \hspace{5mm}\textit{For any $\tau$, there exists a vector bundle $S_\tau(x) \subset \mathbb{R}^M$ such that $\frac{d\varphi}{dx} S_\tau(x) = T^\perp_{\varphi^\tau(x)}\mathcal{M}$ is normal to the manifold.}


    Because $\varphi^\tau$ is a diffeomorphism, $\frac{d\varphi^\tau}{dx}$ is full rank.
    We can construct any vector $v = \left[\frac{df}{dx}\right]^{-1} u$ for some $u$.
    Therefore, there exists some subspace $S(x)$ such that $ T_{p}^\perp\mathcal{M} = \left[\frac{df}{dx}\right]^{-1} S(x)$ where $p = \pi_{\mathcal{M}}(\varphi^\tau(x))$ is the projection of $\varphi^\tau(x)$ onto the manifold.
    \hfill$\blacksquare$

    \hspace{5mm}\textit{Let $\Pi_\mathcal{M}^\perp$ be the orthogonal projection onto the normal space. Then $\lim_{\tau \rightarrow \infty} v_0^\top \left[\frac{d\varphi^\tau}{dx}\right]^\top \Pi_\mathcal{M}^\perp \left[\frac{d\varphi^\tau}{dx}\right] v_0 = 0$.}

    Observe that $\left[\frac{d\varphi^\tau}{dx}\right] v_0$ is the solution of the sensitivity equation $\dot{v} = \left[\frac{df}{dx}\right]v$.
    Let $V(v,\tau) = v^\top \Pi_\mathcal{M}^\perp v$.
    We can compute the time derivative of $V(v,\tau) = v^\top \Pi_\mathcal{M}^\perp v$ and show that the quantity decreases monotonically.
    Note that through the use of Fermi coordinates, we have removed the dependency on $\tau$ in $\Pi_{\mathcal{M}}^\perp$.
    The derivative then becomes 
    \begin{align}
        \dot{V}(v,\tau) &= v^\top \left(\Pi_{\mathcal{M}}^\perp\left.\left[\frac{df}{dx}\right]\right|_{x=x_0} + \left.\left[\frac{df}{dx}\right]\right|_{x=x_0}^\top\Pi_{\mathcal{M}}^\perp\right) v\\
        &= v^\top \left(\Pi_{\mathcal{M}}^\perp\left.\left[\frac{df}{dx}\right]\right|_{x = \pi_{\mathcal{M}}(x_0)} + \left.\left[\frac{df}{dx}\right]\right|^\top_{x = \pi_{\mathcal{M}}(x_0)}\Pi_{\mathcal{M}}^\perp\right) v + v^\top \left(\Pi_{\mathcal{M}}^\perp\Gamma(x) + \Gamma(x)^\top\Pi_{\mathcal{M}}^\perp \right) v,
    \end{align}
    where $\|\Gamma(x)\| < L\epsilon$ by the Lipschitz assumption.
    Because $\Pi_{\mathcal{M}}^\perp$ is an orthogonal projection, $\|\Pi_{\mathcal{M}}^\perp\| = 1$.
    The operator norms can then be used to upper bound the derivative as
    \begin{align}
        \dot{V}(v,\tau)&\leq v^\top \left(\Pi_{\mathcal{M}}^\perp\left.\left[\frac{df}{dx}\right]\right|_{x = \pi_{\mathcal{M}(x)}} + \left.\left[\frac{df}{dx}\right]\right|^\top_{x = \pi_{\mathcal{M}}(x_0)}\Pi_{\mathcal{M}}^\perp\right) v + 2 \epsilon L \|v\|^2 \\
        &= 2 v^\top \Pi_{\mathcal{M}}^\perp \left.\left[\frac{df}{dx}\right]\right|_{x = \pi_{\mathcal{M}(x)}} v + 2 \epsilon L \|v\|^2.
    \end{align}
    Decompose $v$ into $v_\perp \in T^\perp_p \mathcal{M}$ and $v_\parallel \in T_p \mathcal{M}$ such that $v = v_\perp + v_\parallel$ and reorder terms as
    \begin{align}
        \dot{V}(v,\tau) &\leq 2 v^\top \Pi_{\mathcal{M}}^\perp \left.\left[\frac{df}{dx}\right]\right|_{x = \pi_{\mathcal{M}(x)}} (v_\perp + v_\parallel) + 2\epsilon L \|v\|^2\\
        &= 2 v_\perp^\top \left.\left[\frac{df}{dx}\right]\right|_{x = \pi_{\mathcal{M}(x)}} (v_\perp + v_\parallel) + 2\epsilon L \|v\|^2\\
        &= 2 v_\perp^\top \left.\left[\frac{df}{dx}\right]\right|_{x = \pi_{\mathcal{M}(x)}} v_\perp + 2 v_\perp \left.\left[\frac{df}{dx}\right]\right|_{x = \pi_{\mathcal{M}(x)}} v_\parallel + 2\epsilon L \|v\|^2\\
        &\leq 2(\epsilon L - \alpha)\|v_\perp\|^2 + 2\epsilon L \|v_\perp\|\|v_\parallel\| + 2 v_\perp^\top \left.\left[\frac{df}{dx}\right]\right|_{x = \pi_{\mathcal{M}(x)}} v_\parallel,
    \end{align}
    where the final line comes from our eigenvalue assumption and the triangle inequality.

    By recalling that $f$ determines the geometry of the manifold through the trajectories, the rightmost term can be seen to represent the derivative of the normal vector along a curve in the manifold.
    By further considering the osculating circle at the point on the associated curve, or the unique circle which matches at least the first two derivatives \citep{lee_2018}, we can see that the directional derivatives of the normal vectors along the tangent space are zero.
    This is because the derivative of the normal component of the circle is zero along the tangent, i.e. $\frac{d}{dx}\cos(x)|_{x=0} = \sin(0) = 0$.
    Thus
    \begin{equation}
        \dot{V}(v, \tau) \leq 2(\epsilon L - \alpha)\|v_\perp\|^2 + 2\epsilon L \|v_\perp\| \|v_\parallel\|,
    \end{equation}
    which is strictly less than zero if
    \begin{equation}
        \frac{\|v_\perp\|}{\|v_\parallel\|} \geq \frac{\epsilon L}{\alpha - \epsilon L}.
    \end{equation}
    Thus, the normal components of all vectors must decay to zero as $\epsilon \rightarrow 0$, and so $\Pi_{\mathcal{M}}^\perp \left[\frac{d\varphi^\tau}{dx}\right] v_0 \rightarrow 0$ as $\tau \rightarrow \infty$.
    \hfill$\blacksquare$

    Finally, by decomposing the transformation into two distinct functions through the semigroup property, we conclude the overall proof.
    Consider $\varphi^{\tau + T}$, where $T$ is sufficiently large that the distance from $\mathcal{M}$ is at most $\epsilon$.
    Decompose the function as $\varphi^{\tau + T} = \varphi^\tau \circ \varphi^T$ and apply the chain rule as
    \begin{equation}
        \frac{d\varphi^{\tau + T}}{dx} = 
        \frac{d\varphi^{\tau}}{d\varphi^{T}(x)}
        \frac{d\varphi^{T}}{dx}.
    \end{equation}
    From our intermediate results, for any $\delta$ we can always choose an appropriate $\tau$ and $v$ such that 
    \begin{equation}
        \frac{d\varphi^{T}}{dx}v \in T_p \mathcal{M} \qquad \text{ and } \qquad \frac{d\varphi^{\tau}}{d\varphi^T(x)} \frac{d\varphi^{T}}{dx}v < \delta.
    \end{equation}

    We now argue that such $v$ converges to a subspace based on the adjoint sensitivity equation
    \begin{equation}
        \dot u = -\left[\frac{df}{dx}\right]^\top u.
    \end{equation}
    The adjoint sensitivity equation runs backwards in time, identifying the directional derivative of the inverse function.
    We will characterize the limiting nullspace of $\frac{d\varphi^\tau}{dx}$ based on the directional derivatives of $\varphi^{-\tau}$ along the normals, which we have shown becomes the left nullspace of the matrix.

    Recall that we can expand the derivative as the derivative on the manifold and an $\epsilon$ perturbation, i.e.
    \begin{equation}
        \dot u = \left( -\left.\left[\frac{df}{dx}\right]\right|_{x = \pi_{\mathcal{M}}(x)}^\top + \Gamma(x)\right) u
    \end{equation}
    where $\|\Gamma(x)\| < \epsilon$.

    By observing that $\left.\left[\frac{df}{dx}\right]\right|^\top_{x=\pi_\mathcal{M}(x)} u \in T^\perp_{\pi_{\mathcal{M}}(x)}\mathcal{M}$, we see that the dynamics converge such that $u$ cannot rotate out of the normal space as $\epsilon \rightarrow 0$.

    Thus, as $T \rightarrow \infty$, the normal space becomes approximately closed in the adjoint analysis, and the vectors which map onto the normal space become $S(x_0)$.

\end{proof}

\begin{theorem}[Context: See Theorem \ref{thm:weakened}]
    \label{prf:weakened}
    Suppose the state space model in Equations \eqref{eq:ssdyn} and \eqref{eq:ssmodel} with noise satisfying the assumptions in Theorem \ref{thm:rank} converges to a $K$-dimensional smooth manifold such that the Jacobian of the limiting flow is rank $K$.
    Then the linear functional which minimizes the CRLB for prediction in the infinite-horizon average cost formulation is the solution of 
    \begin{equation}
        \argmax_{u: \|u\|_{(\Lambda_\sigma + \munderbar{\Sigma})} = 1} \sum_{i=1}^K \alpha_i \langle v_i, u \rangle_{\munderbar{\Sigma}}^2 \label{eq:weakform}
    \end{equation}
    for some $\{\alpha_i\},$
    where ${v_i}$ is the set of right singular vectors of the Jacobian of the limiting flow around the current state with nonzero singular vectors, $\munderbar{\Sigma}$ is the CRLB for estimating the current state based on the past measurements, $\bar{u} = \frac{u}{\|u\|_{\munderbar{\Sigma}}}$ is a unit-length version of $u$, $\langle v, u \rangle_{\munderbar{\Sigma}} = v^\top \munderbar{\Sigma} u$, and $\|\cdot\|_{\munderbar{\Sigma}}$ is the associated induced seminorm.
\end{theorem}
\begin{proof}
    Apply the singular value decomposition to the Jacobian of the flow as $\left[d \varphi^\tau_{x_0}\right] = U_\tau D_\tau V_\tau^*$.
    For sufficiently large $\tau$, by our low-rank assumption, there are at most $K$ non-zero singular values.
    By Lemma \ref{prop:lipschitz}, the row space $V_\tau^*$ is invariant to $\tau$.
    Thus, we let $V_\tau^* = R_\tau^* V^*$, where $R_\tau$ is some $K \times K$ unitary transformation.

    It follows that
\begin{equation}
    \munderbar{\Sigma}_{\tau} = U_\tau D_\tau R_\tau^* \underbrace{\left[V^* \munderbar{\Sigma}_0 V\right]}_{K \times K} R_\tau D_\tau U_\tau^*.
\end{equation}
Thus, by plugging the CRLB into our average cost objective
\begin{align}
    \lim_{N \rightarrow \infty} \frac{1}{N} \sum_{j=1}^{N} \Tr \left( \munderbar{\Sigma}_{\tau_j}\right) &= \frac{1}{N} \sum_{j=1}^{N} \Tr \left(U_\tau D_\tau R_\tau^* \left[V^* \munderbar{\Sigma}_0 V\right] R_\tau D_\tau U_\tau^* \right)\\
    &= \frac{1}{N} \sum_{j=1}^{N} \Tr \left(D_\tau R_\tau^* \left[V^* \munderbar{\Sigma}_0 V\right] R_\tau D_\tau \right).
\end{align}
Using the Sherman--Morrison formula \citep{sherman} for our Fisher Information update, i.e. 
\begin{equation}
    \left[\munderbar{\Sigma}_{-1} + \frac{1}{\|u_i\|_{\Lambda_\sigma}^2} u_i u_i^\top\right]^{-1} = \munderbar{\Sigma}_{-1} - \frac{\munderbar{\Sigma}_{-1} u_i u_i^\top \munderbar{\Sigma}_{-1}}{\|u_i\|_{\Lambda_\sigma}^2 + u_i^\top \munderbar{\Sigma}_{-1} u_i},
\end{equation}
the update to the future cost becomes
\begin{equation}
    \frac{-1}{N\left(\|u_i\|_{\Lambda_\sigma}^2 + \|u_i\|_{\munderbar{\Sigma}_{-1}}^2\right)} \sum_{j=1}^{N} \Tr \left(D_\tau R_\tau^* \left[V^* \munderbar{\Sigma}_{-1} u_i u_i^\top \munderbar{\Sigma}_{-1} V\right] R_\tau D_\tau \right).
\end{equation}
Next, note that by expressing terms as
\begin{equation}
    V^* \munderbar{\Sigma}_{-1} u_i = \begin{bmatrix} \langle v_1, u_i \rangle_{\Sigma_{-1}} \\ \vdots \\ \langle v_K, u_i \rangle_{\Sigma_{-1}} \end{bmatrix} \qquad R_\tau D_\tau D_\tau R_\tau^* = \begin{bmatrix}
        \alpha_{1,1}^{(\tau)} & \alpha_{1,2}^{(\tau)} & \cdots & \alpha_{1,K}^{(\tau)} \\
        \alpha_{2,1}^{(\tau)} & \alpha_{2,2}^{(\tau)} & \cdots & \alpha_{2,K}^{(\tau)} \\
        \vdots & \vdots & \ddots & \vdots \\
        \alpha_{K,1}^{(\tau)} & \alpha_{K,2}^{(\tau)} & \cdots & \alpha_{K,K}^{(\tau)}
\end{bmatrix},
\end{equation}
we can rewrite the expression in the trace as
\begin{equation}
    \Tr \left(D_\tau R_\tau^* \left[V^* \munderbar{\Sigma}_{-1} u_i u_i^\top \munderbar{\Sigma}_{-1} V\right] R_\tau D_\tau \right) = \sum_{k=1}^K \alpha_{k,k}^{(\tau)} \langle v_k, u_i \rangle_{\munderbar{\Sigma}_{-1}}^2.
\end{equation}
Rewrite the summation in terms of our new expression for the trace and commute the summations as
\begin{equation}
    \argmax_{u_i} \sum_{k=1}^K\left(\frac{\langle v_k, u_i \rangle^2_{\munderbar{\Sigma}_{-1}}}{\|u_i\|_{\Lambda_\sigma}^2 + \|u_i\|_{\munderbar{\Sigma}_{-1}}^2}  \left(\lim_{N \rightarrow \infty} \frac{1}{N} \sum_{j=1}^N \alpha_{k,k}^{(\tau)}\right)\right)
\end{equation}
Finally, let $\alpha_k := \lim_{N \rightarrow \infty} \frac{1}{N} \sum_{j=1}^N \alpha_{k,k}^{(\tau)}$ to express the objective as
\begin{equation}
    \argmax_{u_i} \frac{\sum_{k=1}^K\alpha_k\langle v_k, u_i \rangle^2_{\munderbar{\Sigma}_{-1}}}{\|u_i\|_{\Lambda_\sigma}^2 + \|u_i\|_{\munderbar{\Sigma}_{-1}}^2}. 
\end{equation}
Observe that $\|u_i\|_{\Lambda_\sigma}^2 + \|u_i\|_{\munderbar{\Sigma}_{-1}}^2 = \|u_i\|^2_{\Lambda_\sigma + \munderbar{\Sigma}_{-1}}$ and apply bilinearity to move it inside the inner products as
\begin{equation}
    \argmax_{u_i} \sum_{k=1}^K \alpha_k \left\langle v_k, \frac{u_i}{\|u_i\|_{(\Lambda_\sigma + \munderbar{\Sigma}_{-1})}} \right\rangle^2_{\munderbar{\Sigma}_{-1}}. 
\end{equation}
Finally, replace the normalization with a constraint that $\|u_i\|_{(\Lambda_\sigma + \munderbar{\Sigma}_{-1})} = 1$, noting that by the absolute homogeneity of norms, the result is invariant to scaling:
\begin{equation}
    \argmax_{u_i: \|u_i\|_{(\Lambda_\sigma + \munderbar{\Sigma}_{-1})} = 1} \sum_{k=1}^K \alpha_k \left\langle v_k, u_i \right\rangle^2_{\munderbar{\Sigma}_{-1}}. 
\end{equation}
\end{proof}

\begin{corollary}[Proof: See Corollary \ref{cor:basis}]
    \label{prf:basis}
    Under the assumptions of Theorem \ref{thm:weakened}, the optimal measurement vector exists in the subspace
    \begin{equation}
        \textup{Span}\,\left\{(\munderbar{\Sigma} + \Lambda_\sigma)^{-1} \munderbar{\Sigma} v_i \right\}_{i=1}^K.
    \end{equation}
\end{corollary}
\begin{proof}
    This statement follows quickly from the application of Lagrange multipliers.
    Begin with the statement form Theorem \ref{thm:weakened},
    \begin{equation}
        \argmax_{u: \|u\|_{(\Lambda_\sigma + \munderbar{\Sigma})} = 1} \sum_{i=1}^K \alpha_i \langle v_i, u \rangle_{\munderbar{\Sigma}}^2 \label{eq:weakform}.
    \end{equation}
    Introduce the construct the Lagrangian
    \begin{equation}
        \mathcal{L}(u, \lambda) = \sum_{i=1}^K \alpha_i \langle v_i, u \rangle^2_{\munderbar{\Sigma}} + \lambda (1 - \|u\|_{(\Lambda_\sigma + \munderbar{\Sigma})}^2),
    \end{equation}
    where the norm has been squared, leaving the constraint unchanged.

    Expand the expressions
    \begin{equation}
        \mathcal{L}(u, \lambda) = \sum_{i=1}^K \alpha_i \left(v_i^\top \munderbar{\Sigma} u \right)^2 + \lambda (1 - u^\top \left(\Lambda_\sigma + \munderbar{\Sigma}\right)u).
    \end{equation}

    Compute the derivative with respect to $u$: 
    \begin{equation}
        \frac{d\mathcal{L}}{du} = \sum_{i=1}^K \alpha_i \left(v_i^\top \munderbar{\Sigma} u \right) \munderbar{\Sigma} v_i - 2 \lambda \left(\Lambda_\sigma + \munderbar{\Sigma}\right)u.
    \end{equation}

    Set equal to zero and rearrange terms to find
    \begin{equation}
        u = \frac{1}{2\lambda}\sum_{i=1}^K \underbrace{\alpha_i \left(v_i^\top \munderbar{\Sigma} u \right)}_{\in \mathbb{R}}\left(\Lambda_\sigma + \munderbar{\Sigma}\right)^{-1}\munderbar{\Sigma} v_i.
    \end{equation}

    Finally, observe that $\alpha_i \left(v_i^\top \munderbar{\Sigma} u \right) \in \mathbb{R}$, and thus
    \begin{equation}
        u \in \textup{Span}\,\left\{(\munderbar{\Sigma} + \Lambda_\sigma)^{-1} \munderbar{\Sigma} v_i \right\}_{i=1}^K.
    \end{equation}
\end{proof}

\begin{corollary}[Proof: See Corollary \ref{cor:closed}]
    \label{prf:closed}
    If, in addition to the assumptions of Theorem \ref{thm:weakened}, the dynamical system converges to an isolated limit cycle, then 
    \begin{equation}
        u = (\munderbar{\Sigma} + \Lambda_\sigma)^{-1} \munderbar{\Sigma} v
    \end{equation}
    is an optimal measurement vector, where $v$ is the right singular vector.
\end{corollary}
\begin{proof}
    By Corollary \ref{cor:invariant}, the information content is invariant to the scaling of the linear functional.
    By Corollary \ref{cor:basis}, the optimal measurement functional is in the span of $(\munderbar{\Sigma} + \Lambda_\sigma)^{-1} \munderbar{\Sigma} v$. 
    Thus, the vector generating the subspace is one possible optimal measurement.
\end{proof}

\section{COMPUTATIONAL COSTS}
\label{app:computation}

In this section, we detail some of the specific aspects of the computation in this work.
In particular, we detail the computational costs of the various operations, as well as provide more detail on the gradient ascent procedure for optimizing the 1D approximation.
In Table \ref{tab:costs} we include a summary of the computational cost scaling for the dynamic programming approach and the proposed gradient ascent based dimensionality reduction approach.
Note that, as described in Section \ref{app:advance}, the evaluation cost of the proposed method can be reduced to $O(C M^2)$, where $C$ is the number of gradient ascent steps, when the governing dynamics are particularly easy to solve.

\begin{table*}[ht]
    \vskip 0.15in
    \begin{center}
    \begin{small}
    \begin{sc}
    \begin{tabular}{|c|cc|}
    \toprule
       & Initialization & Evaluation \\
    \midrule
    Dynamic Programming     & $O\left((MN^2 + NKP)|\mathcal{A}|\right)$  & $O\left((M^3 + K)|\mathcal{A}|\right)$  \\
    Proposed Method            & $O(1)$                                    & $O(M^2C + M^3)$  \\
    \bottomrule
    \end{tabular}
    \end{sc}
    \end{small}
    \end{center}
    \vskip -0.1in
    \caption{Multiplication scaling for the two approaches outlined in this work, where $M$ is the dimensionality of the system, $N$ is the number of samples in the state space approximation, $K$ is the average number of state space representation points in a $d_{max}$ ball, $P$ is the number of steps in value iteration, $|\mathcal{A}|$ is the number of possible actions, and $C$ is the number of gradient ascent steps.}
    \label{tab:costs}
\end{table*}

From this small table, the main observation is that the dynamic programming approach requires a prohibitively expensive initialization, as $N$ is typically significantly larger than any other variables in the expressions.
Furthermore, we do not necessarily incur a significant cost in the proposed method for evaluating the policy.
Differences in iteration cost essentially reduce to a comparison of the number of gradient ascent steps $C$ to the product of the number of discrete actions and the dimensionality of the system $M|\mathcal{A}|$. 
As the surface area of the hypersphere grows exponentially with the dimensionality, the latter is often significantly larger.

\subsection{Sampling State Space}

$O(NM^3)$ multiplications, where $N$ is the total number of samples.

First, we must generate $M^2 + M$ scalar random variables.
Then, QR factorization requires $O(M^3)$ multiplications \citep{trefethen_1997}.
As $N$ samples are required, the entire cost is $O(NM^3)$ multiplications.

\subsection{Advancing State and CRLB in Time}
\label{app:advance}

Worst Case: $O(M^3)$ multiplications; Best Case: $O(M^2)$ multiplications

As evaluating $\varphi^\tau$ and $d\varphi^\tau$ are heavily dependent on the particular system, we treat these as some fixed constant cost, though it ranges from $O(M)$ to $O(M^3)$.

The remainder of the update is then $\left[d \varphi^\tau_x\right]\left(\munderbar{\Sigma}^{-1} + \sigma^{-2} u u^\top\right)^{-1}\left[d \varphi^\tau_x\right]^\top$, where the inverse is computed using the Sherman--Morrison Formula \citep{sherman} 
\begin{equation}
\left(\munderbar{\Sigma}^{-1} + \sigma^{-2} uu^\top \right)^{-1} = \munderbar{\Sigma} - \frac{\munderbar{\Sigma}uu^\top \munderbar{\Sigma}}{\sigma^2 + u^\top \munderbar{\Sigma}u}.
    \label{eq:update_app}
\end{equation}

The number of multiplications for equation \eqref{eq:update_app} can be limited through the following order of computations.
\begin{enumerate}
    \item Compute $\munderbar{\Sigma} u$ for $M^2$ multiplications
    \item Compute $u^\top (\munderbar{\Sigma} u)$ for $M$ multiplications
    \item Compute $(\munderbar{\Sigma} u)/(\sigma^2 + u^\top \munderbar{\Sigma} u)$ for $M$ multiplications
    \item Compute $\left[(\munderbar{\Sigma} u)/(\sigma^2 + u^\top \munderbar{\Sigma} u)\right]\left[\munderbar{\Sigma} u\right]^\top$ for $M^2$ multiplications.
\end{enumerate}

Thus, the update equation requires $2(M^2 + M)$ multiplications.

\subsection{Computation of Local Average Weights}

$O(MN^2|\mathcal{A}|)$ multiplications, where $|\mathcal{A}|$ is the cardinality of the finite action space.

This involves computing the distance between all potential states after the observation at a given sample to all samples.
Thus, as $\|x - x'\|$ requires $M+1$ multiplications, the computation of the entire set of distances requires $MN^2|\mathcal{A}|$ multiplications, where $|\mathcal{A}|$ is the cardinality of the finite action space.

It is worth noting that this represents the primary cost of the dynamic programming approach, as the required value of $N^2$ grows rapidly with the dimensionality of the system.

\subsection{Value Iteration}

$O(|\mathcal{A}|NKP)$ multiplications, where $K$ is the average number number of points in a $d_{max}$ ball, and $P$ is the number of steps in value iteration.

In this section, assume $d_{max}$ is appropriately scaled such that states have an average number of $K \ll N$ neighbors.
Then, each action for each state sample requires on average $K$ multiplications, resulting in a total of $|\mathcal{A}|NK$ multiplications.

\subsection{Dynamic Programming Optimal Action Evaluation}

$O((M^3 + K)|\mathcal{A}|)$ multiplications

For each possible action, we must advance the state and CRLB forward in time by one timestep, requiring $O(M^3)$ multiplications. 
We then must compute the local average of the sample value function for $K$ multiplications.

\subsection{Extended Kalman Filter}

$O(M^2)$ multiplications

\begin{align}
    \hat{x}_{i|i} &= \hat{x}_{i|i-1} +  \frac{\munderbar{\Sigma}_i u_i(y_i - u_i^\top \hat{x}_{i|i-1})}{u_i^\top \munderbar{\Sigma}_i u_i + \sigma^2},
\end{align}

\begin{enumerate}
\item Compute $\munderbar{\Sigma} u$ for $M^2$ multiplications.
\item Compute $u^\top \left(\munderbar{\Sigma} u\right)$ for $M$ multiplications.
\item Compute $u^\top \hat{x}_{i|i-1}$ for $M$ multiplications
\item Combine multiplicative constants for $1$ multiplication.
\item Remaining update for $M$ multiplications.
\end{enumerate}

Adding the above costs shows that the EKF update step requires $M^2 + 3M + 1$ multiplications.

\subsection{Gradient Ascent Optimization}

$O(M^2)$ multiplications per gradient step

We now address the details of the iterative optimization of
\begin{equation}
    \argmax_{u: \|u\|=1} \frac{(v^\top \munderbar{\Sigma} u)^2 }{u^\top \munderbar{\Sigma} u + \sigma^2u^\top u}.
\end{equation}
To compute the gradient of the function on the sphere, we first need to compute the gradient in the ambient Euclidean space, then project onto the tangent space.

By the application of standard differentiation rules
\begin{align}
    d \left(\frac{(v^\top \munderbar{\Sigma} u)^2 }{u^\top \munderbar{\Sigma} u + \sigma^2u^\top u}\right)_u &=
    2 \frac{\left(u^\top \munderbar{\Sigma} u + \sigma^2u^\top u\right) v^\top \munderbar{\Sigma} u v^\top \munderbar{\Sigma} - \left(v^\top \munderbar{\Sigma} u\right)^2 (u^\top \munderbar{\Sigma} + u^\top)}{\left(u^\top \munderbar{\Sigma} u + \sigma^2u^\top u\right)^2} \\
    &= 2 \frac{\left(u^\top \munderbar{\Sigma} u + \sigma^2\right) v^\top \munderbar{\Sigma} u v^\top \munderbar{\Sigma} - \left(v^\top \munderbar{\Sigma} u\right)^2 (u^\top \munderbar{\Sigma} + u^\top)}{\left(u^\top \munderbar{\Sigma} u + \sigma^2\right)^2},
\end{align}
where the second line comes from noting that $\|u\| = 1$.
Finally, we can slightly reduce the computation further by noting that we eventually orthogonalize the gradient to $u$. 
Thus, the following is an equivalent computation
\begin{equation}
     2 \frac{\left(u^\top \munderbar{\Sigma} u + \sigma^2\right) v^\top \munderbar{\Sigma} u v^\top \munderbar{\Sigma} - \left(v^\top \munderbar{\Sigma} u\right)^2 u^\top \munderbar{\Sigma} }{\left(u^\top \munderbar{\Sigma} u + \sigma^2\right)^2},
\end{equation}

By performing operations in the following order, the differential of the value in Euclidean space requires few operations.
Below, we count the number of multiplications.
\begin{enumerate}
\item Compute $\munderbar{\Sigma} u$ for $M^2$ multiplications.
\item Compute $u^\top \left(\munderbar{\Sigma} u\right)$ for $M$ multiplications.
\item Compute $v^\top \left(\munderbar{\Sigma} u\right)$ for $M$ multiplications.
\item Compute $v^\top \munderbar{\Sigma}$ for $M^2$ multiplications. 
\item Computing remaining terms in numerator for $2M + 2$ multiplications.
\item Divide by numerator for $M + 1$ multiplications.
\end{enumerate}

Finally, we must project the result onto the tangent space of the sphere.
Fortunately, the tangent space is straightforward for a sphere, in particular, it is the space orthogonal to $u$.
Thus, we project the vector $u$ out with $2M$ additional multiplications.

By adding all of the above, we find that computation of the gradient requires $2M^2 + 7M + 3$ multiplications.

Once the gradient has been computed, it defines the discrete-time approximation of the gradient flow by following the geodesic curve
\begin{equation}
    u_{i+1} = \cos(\alpha_i \|s_i\|)u_i + \sin(\alpha_i \|s_i\|)s_i/\|s_i\|,
\end{equation}
where $s_i$ is the gradient in the ambient space.
The above computation requires another $3M$ multiplications.

Thus, each gradient ascent step requires $2M^2 + 10M$ plus some constant number of multiplications from evaluating sine and cosine.

\subsubsection{Gradient Ascent Simulations}

To give a sense of the non-linear optimization problem, we include visualizations of two potential optimization surfaces for 3D systems, shown as two different panels in Figure \ref{fig:optsurface}.
Because the domain of the function is the unit sphere, we include a pair of heatmaps in each panel representing the front and back of the sphere by a z-coordinate threshold to show the occluded side.

\begin{figure}[ht]
    \centering
    \includegraphics[width=0.95\textwidth]{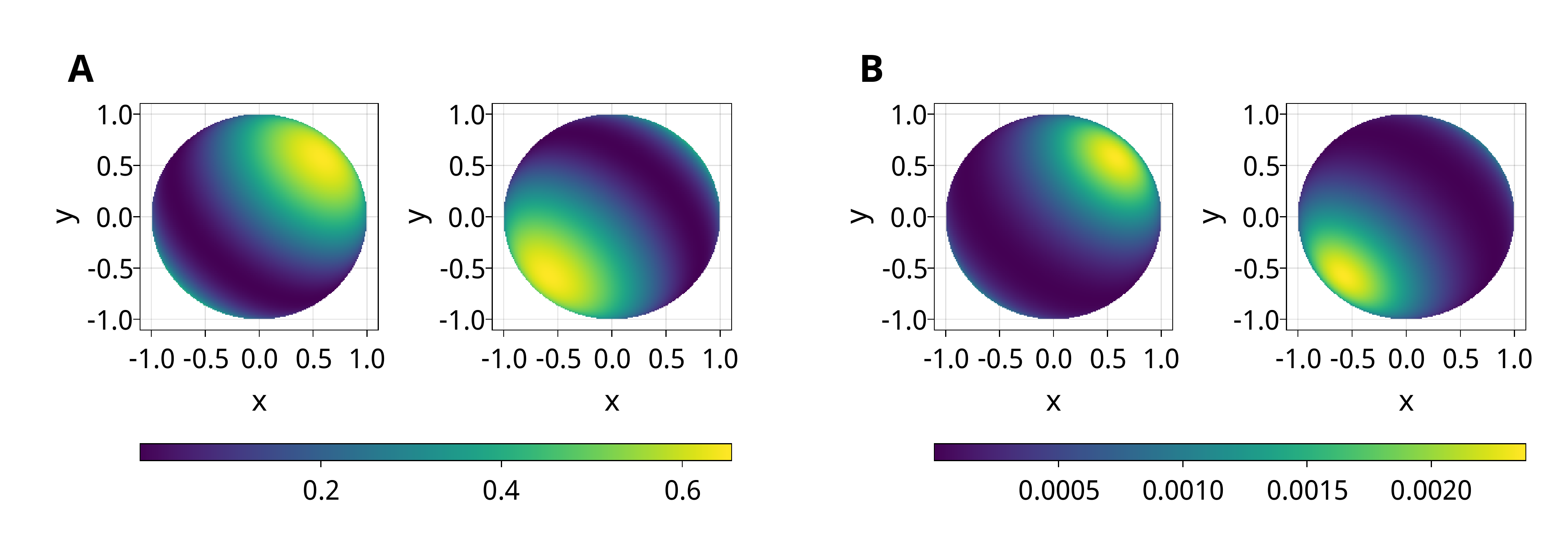}
    \caption{Optimization surface for gradient ascent for two different estimation CRLBs. In Panel B, the estimation CRLB is chosen to be poorly aligned with the limit vector.}
    \label{fig:optsurface}
\end{figure}

There are a two useful observations to be made.
First, as was clear from both intuition and the form of the optimization problem, the function is symmetrical.
Next, in these particular simulations, the function is relatively smooth with a ``unique'' maximum, only mirrored on the opposite side.
This second observation allows the gradient ascent to be rather straightforward and gives us some certainty that, given appropriate step sizes, we will identify the true optimum.
In these plots, $v = [3^{-1/2}, 3^{-1/2}, 3^{-1/2}]$ and $\sigma^2 = 1$.
The estimation CRLB took the form:
\begin{equation}
\text{\textbf{Panel A:} }\quad \munderbar{\Sigma} = \begin{bmatrix}
1.0 & 0.1 & 0.1 \\
0.1 & 1.0 & 0.1 \\
0.1 & 0.1 & 1.0 
\end{bmatrix}
\qquad\qquad 
\text{\textbf{Panel B:} }\quad \munderbar{\Sigma} = \begin{bmatrix}
1.0 & 0.0 & 0.0 \\
0.0 & 1.0 & 0.0 \\
0.0 & 0.0 & 1.0 
\end{bmatrix} - 0.95 v v^\top.
\end{equation}

Figure \ref{fig:trajectory} provides a representative trajectory of the gradient ascent operation over the previously illustrated objective.
\begin{figure}[ht]
    \centering
    \includegraphics[width=0.6\textwidth]{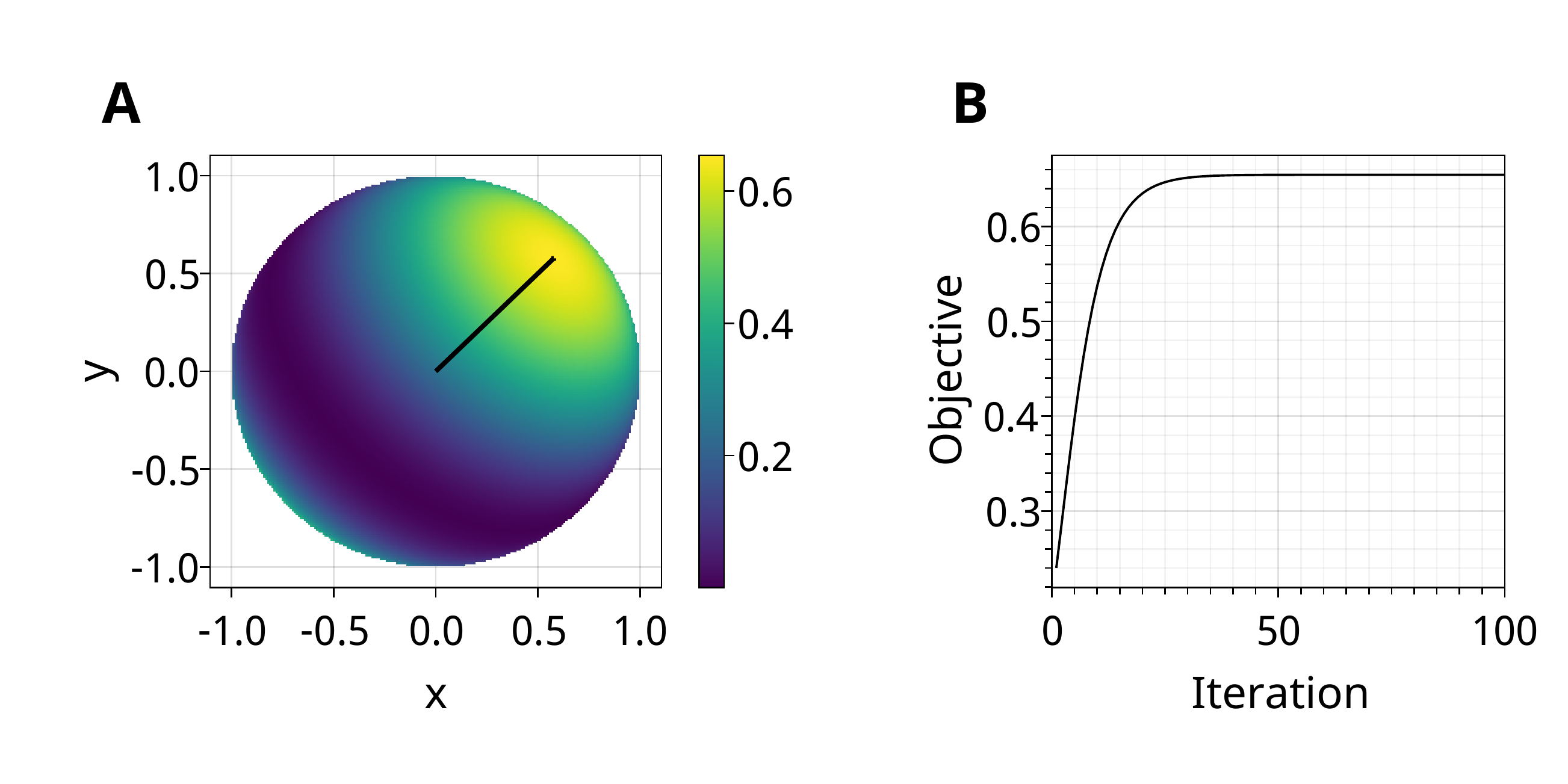}
    \caption{Example gradient ascent trajectory on our objective function. Panel A: Objective function on sphere with black line corresponding to the ascent trajectory; Panel B: Objective function value as a function of iteration.}
    \label{fig:trajectory}
\end{figure}

In Panel A of Figure \ref{fig:trajectory}, we plot the trajectory of the gradient ascent on top of the objective function.
The initialization was chosen to emphasize that the trajectory approximately follows a geodesic: when the projection of the sphere onto a plane is centered on the initial position, geodesic curves appear straight, similar to the behavior in this plot.
Panel B shows the evaluation of the objective function over time, illustrating the rapid convergence in this simulation.
In this simulation, a constant step size was chosen of $\alpha = 0.1$.

\subsection{Timing Simulations}

While the implementations in this work are not optimized, we include some timing plots below to give a sense of the rate at which the computational costs grow for the different approaches and provide a general order of magnitude for the performance.
In our timing plots, a Van der Pol system is being observed.
In order to best elucidate the scaling aspects, in the proposed method, the computation time is plotted as a function of dimensionality, while the dynamic programming computation time is plotted as a function of the number of state space representation points.
The timing plots are available in Figure \ref{fig:timingPlot}.

\begin{figure}[ht]
    \centering
    \includegraphics[width=0.6\textwidth]{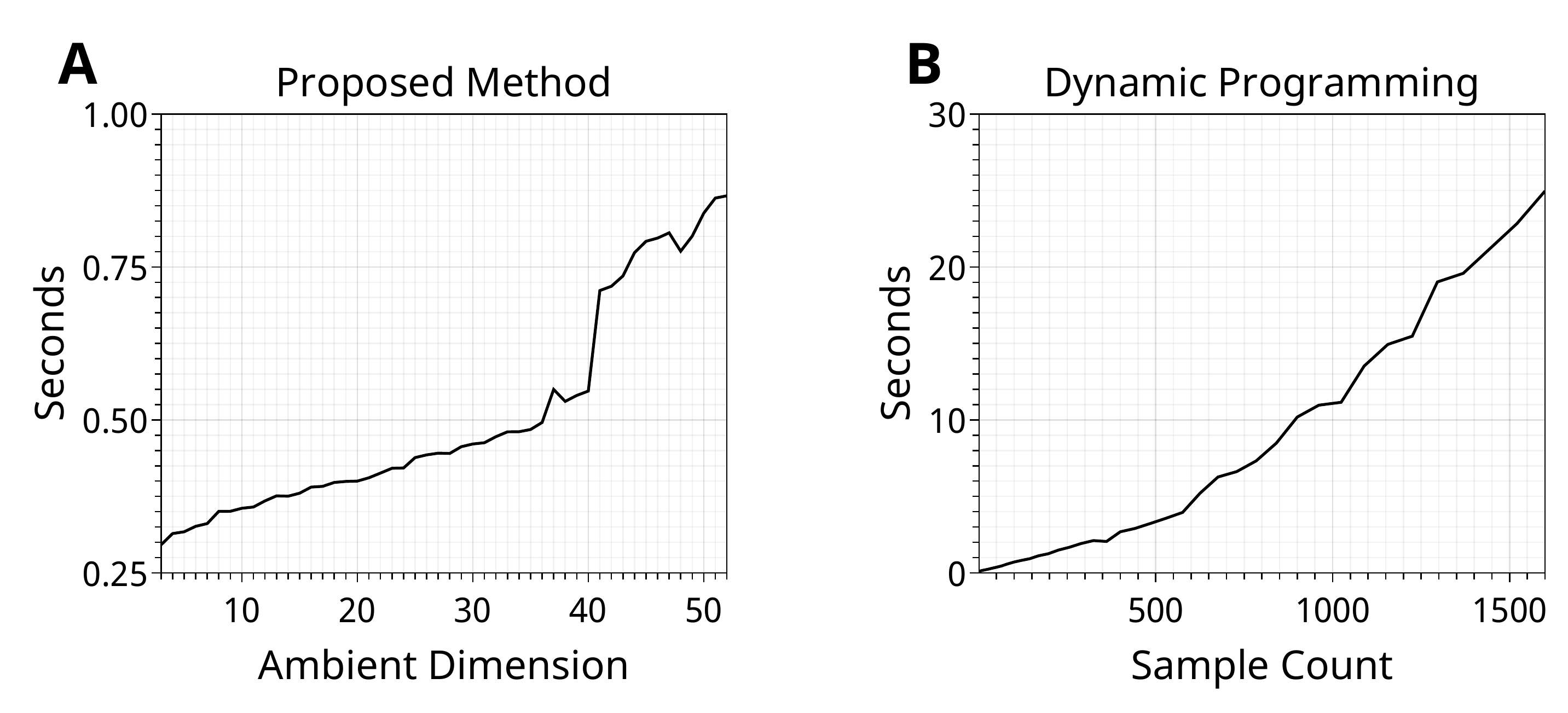}
    \caption{A: Timing of proposed method as a function of dimensions. B: Timing of the dynamic programming baseline as a function of the number of samples. Note that the required samples may grow exponentially with the dimensionality of the space, resulting in very fast growth of the computational costs.}
    \label{fig:timingPlot}
\end{figure}

To generate the plots, the time required to initialize the algorithms as well as compute 1000 decision steps was recorded.
Timing was done using the default number of runs in the BenchmarkTools library for Julia, which returns the minimum time of multiple runs to eliminate issues relating to just-in-time compilation and unrelated computer processes interfering with a run.

Timesteps were spaced by $0.05$ units of time in both simulations.
In the proposed method, the 1D approximation was made at $T=10$ units of time in the future, and $100$ gradient ascent steps were computed.
In the dynamic programming approach, the cardinality of the action space was 31, $d_{max} = 1$, and $100$ steps of value iteration were completed in the initialization.

In the simulation of the proposed method shown in Panel A of Figure \ref{fig:timingPlot}, we observe nearly linear growth in the computation time as a function of the dimensionality of the system with a sudden change around 40 dimensional systems potentially resulting in a change in caching at that point.
This suggests that for the chosen parameters, the cost is dominated by system simulation costs which grow linearly for the chosen system as a function of dimension.
In Panel B, the growth is clearly faster than linear.

It is additionally worth noting that, for reasonable performance, Panel B represents only a small subset of the domain of Panel A.
As previously shown in Figure \ref{fig:distances}, the required number of samples rapidly increases with the dimensionality of the system, resulting in the dynamic programming approach becoming intractable while the proposed method still requires less than 1 second of computation.

\subsection{System Specifications}
\label{sec:systemspec}

The code in this project was run on on an AMD Ryzen Threadripper 3960X with 64 GB of RAM and 8GB swap running Ubuntu 20.04.2 LTS.
All code was written in Julia 1.7.1, using the libraries
\begin{enumerate}
    \item DifferentialEquations v7.0.0, MIT "Expat" License, \citep{rackauckas2017differentialequations}
    \item DiffEqSensitivity v6.69.0, MIT "Expat" License, \citep{ma_comparision}
\item EllipsisNotation v1.1.3, MIT "Expat" License
\item ForwardDiff v0.10.25, MIT License, \citep{RevelsLubinPapamarkou2016}
\item SafeTestsets v0.0.1, MIT "Expat" License
\item Documenter v0.27.12, MIT License
\item CSV v0.10.2, MIT "Expat" License
\item BenchmarkTools v1.2.2, MIT "Expat" License, \citep{BenchmarkTools.jl-2016}
\item CairoMakie v0.6.4, MIT "Expat" License, \citep{DanischKrumbiegel2021}
\end{enumerate}

\section{NUMERICAL ISSUES}
\label{app:numerical}
While the approximately singular behavior in this work benefit the statistical properties, they can introduce numerical issues in the implementation.
While the algorithm introduced in the main paper functions in exact arithmetic, rounding errors in floating point arithmetic result in matrices becoming singular.
In this section, we discuss the required modifications to account for these numerical issues relating to inverting such matrices.

The main issue is from the CRLB update formula
\begin{equation}
    \munderbar{\Sigma}_{i+1} = \left[d\varphi_x^\tau\right] \left(\munderbar{\Sigma}^{-1}_{i} + \sigma^{-2}u u^\top\right)^{-1}\left[d\varphi_x^\tau\right]^\top.
\end{equation}
Recalling that $\munderbar{\Sigma}$ and $\munderbar{\Sigma}^{-1}$ both begin as positive definite matrices, consider the operations that could make the inverse impossible.
First, $\munderbar{\Sigma}^{-1} + \sigma^{-2}u u^\top$ may only become singular if $\sigma^{-2}uu^\top$ is significantly large compared to an eigenvalue of the Fisher information.
Fortunately, due to the dimensionality collapse and the continual collection of data, this issue does not appear.

The other, more likely, issue is that the application of the Jacobian creates a singular matrix.
In practice, this happens quite often and thus must be addressed by replacement with an equivalent computation.

To resolve these issues, we complete the updates through the Sherman--Morrison formula.
We include here a brief note on the behavior of the Sherman--Morrison formula when a matrix is singular.

\begin{equation}
    \left(\munderbar{\Sigma}^{-1} + \sigma^{-2} uu^\top \right)^{-1} = \munderbar{\Sigma} - \frac{\munderbar{\Sigma}uu^\top \munderbar{\Sigma}}{\sigma^2 + u^\top \munderbar{\Sigma}u}. \label{eq:inverse}
\end{equation}

We argue that, when $\munderbar{\Sigma}$ is singular, this update formula is equivalent to updating the subspace containing non-zero eigenvalues.
To see this, we simply replace $\munderbar{\Sigma}^{-1}$ with the Moore-Penrose pseudoinverse before multiplying the terms to verify the property.
Thus, replacement with the pseudo-inverse produces a new pseudo-inverse in the formula.
By recalling that the pseudoinverse can be computed by inverting all non-zero singular values, we can conclude that the formula operates correctly in the subspace.

\section{CODE USAGE}
\label{app:code}
The library produced for this manuscript is publicly available under the MIT license at \verb|github.com/Helmuthn/naumer_Dimensionality_2022.jl|.

The implementation is a Julia package, and while not registered, can be installed through the package manager in the REPL through
\verb|] add https://github.com/Helmuthn/naumer_Dimensionality_2022.jl|.
We include here a short summary of the interface which much be defined to apply the library to a new problem.

The functions in the library make use of the type \verb|AbstractSystem|, which is an abstract type with the following three functions:

\verb|dimension|

This function accepts an instance of \verb|AbstractSystem| and returns the ambient dimensionality.

\verb|flow|

This function accepts an instance of \verb|abstractSystem|, a state vector $x$, and a duration of time $\tau$ and returns the state advanced by $\tau$.

\verb|flowJacobian|

This function accepts an instance of \verb|abstractSystem|, a state vector $x$, and a duration of time $\tau$ and returns the Jacobian of the flow advancing by $\tau$ around initial condition $x$.

While trivial, we now include the complete implementation of \verb|LinearSystem|, which represents a linear system of ordinary differential equations, below as an example of an implementation of the interface.

\fakesection{LinearSystem}

\begin{lstlisting}[language=Julia]
"""
    LinearSystem{T} <: AbstractSystem{T}
Defines a linear system dx/dt = Ax
### Fields
 - `dynamics::Matrix{T}` - Matrix defining dynamics
"""
struct LinearSystem{T} <: AbstractSystem{T}
    dynamics::Matrix{T}
end

function flow(x::Vector, τ, system::LinearSystem)
    return exp(τ * system.dynamics) * x
end

function flowJacobian(x::Vector, τ, system::LinearSystem)
    return exp(τ * system.dynamics)
end

function dimension(system::LinearSystem)
    return size(system.dynamics)[1]
end
\end{lstlisting}

In typical applications, there will be no closed form solution for the dynamics.
When this is the case, one can use standard libraries such as \verb|DifferentialEquations.jl| and \verb|ForwardDiff.jl| to define the AbstractSystem.
An example is included below.

\fakesection{VanDerPolSystem}

\begin{lstlisting}[language=Julia]
using DifferentialEquations, ForwardDiff

struct VanDerPolSystem{T} <: AbstractSystem{T}
    μ::T
end

function dimension(system::VanDerPolSystem)
    return 2
end

function vanderpolDynamics!(du, u, p, t)
    x, y = u
    μ = p
    du[1] = y
    du[2] = μ * (1 - x^2) * y - x
end

function flow(x::Vector, τ, system::VanDerPolSystem)
    problem = ODEProblem(vanderpolDynamics!, x, (0.0, τ), system.μ)
    sol = solve(problem, Tsit5(), reltol=1e-8)

    return sol[end]
end

function flowJacobian(x::Vector, τ, system::VanDerPolSystem)
    problem = ODEProblem(hopfDynamics!, x, (0.0, τ), system.μ)

    function solvesystem(init)
        prob = remake(problem, u0=init)
        sol = solve(prob, Tsit5(), reltol=1e-8)
        return sol[end]
    end

    return ForwardDiff.jacobian(solvesystem, x)
end
\end{lstlisting}

Discrete-time systems can be defined by restricting $\tau$ in the function definitions, i.e.

\begin{lstlisting}[language=Julia]
flow(x::Vector, τ::Integer, system::AbstractSystem)
\end{lstlisting}

\fakesection{DiscreteSystem}

\begin{lstlisting}[language=Julia]
struct DiscreteSystem{T} <: AbstractSystem{T}
    dynamics::Matrix{T}
end

function flow(x::Vector, τ::Integer, system::DiscreteSystem)
    out = copy(x)
    for i in 1:τ
        out .= system.dynamics * out
    end
    return out
end

function flowJacobian(x::Vector, τ::Integer, system::DiscreteSystem)
    return system.dynamics^τ
end

function dimension(system::DiscreteSystem)
    return size(system.dynamics)[1]
end
\end{lstlisting}

\section{SIMULATION DETAILS}
\label{app:details}
In this section, we provide all parameters used in the generation of the figures in the main work.
All code used to generate the figures is additionally available in the previously mentioned repository.

\subsection{Figure \ref{fig:DimensionCollapse}}
Simulations regarding the Van der Pol oscillator with parameter $\mu=1$ were completed for this figure.

In Panel A, the direction of the flow was plotted on a grid with samples spaced by approximately $1/3$.

In Panel B, four trajectories initialized at
\begin{equation}
x_0 = \begin{bmatrix}
    0.0 \\ -4.0
\end{bmatrix} \qquad
x_0 = \begin{bmatrix}
    0.0 \\ -0.1
\end{bmatrix} \qquad
x_0 = \begin{bmatrix}
    0.0 \\ 0.1
\end{bmatrix} \qquad
x_0 = \begin{bmatrix}
    0.0 \\ 4.0
\end{bmatrix},
\end{equation}
and were simulated for 20 seconds of system time.

In Panel C, systems initialized at each of the grid points from Panel A were simulated for 20 seconds of system time with relative and absolute tolerance of $10^{-6}$.
We then numerically differentiated the solution using the ForwardDiff library, a common forward mode automatic differentiation library in Julia with good support for differential equations sensitivity analysis.
The singular value decomposition of the resulting Jacobian was computed, and the right singular vector associated with the largest singular value was plotted. 

In Panel D, to compute a representation of the limit cycle, the system was initialized at 
\begin{equation}
x_0 = \begin{bmatrix} 1.0 & 1.0 \end{bmatrix}^\top
\end{equation}
and was then simulated for 20 seconds of system time.
We then computed the squared distances between the final point of the trajectory and the previous 100 positions spaced by 0.01 seconds of system time.
As the period of the limit cycle is between 1 and 2 seconds of system time, we then computed the nearest position in the list.
The trajectory from that index until the end was used to represent the phase of the limit cycle.

At each of the grid points from Panel A, the system trajectory was simulated for 40 seconds of system time.
It was then assigned the index of the closest point in the phase representation of the limit cycle as its value and plotted as a heatmap with contours.

\subsection{Figure \ref{fig:SystemSampling}}
\subsubsection{Panel A}
Random sampling was completed through MersenneTwister in the Random library of Julia with the seed 1234, chosen arbitrarily.

The system took a constant value of $x = \begin{bmatrix}1 & 1 \end{bmatrix}^\top$.

As the system trajectory value does not impact the CRLB updates in this system, 1000 PSD matrix samples were used with eigenvalues following exponential distributions with parameter $\lambda=1$, while only one state space sample was chosen.
The maximum distance for averaging was chosen as $d_{max}=1$.

Observations of the system were made under i.i.d. Gaussian noise with variance $\sigma^2 = 1$.
The action space for both the dynamic programming baseline approach and the proposed method was the set of vectors along the unit circle spaced by $0.1$ radians between $0$ and $\pi$.

The CRLB was initialized as
\begin{equation}
    \munderbar{\Sigma} = \begin{bmatrix} 4 & 0 \\
    0 & 4 \end{bmatrix},
\end{equation}
while the state was randomly initialized by sampling two i.i.d. Gaussian random variables with variance $4$.

Value iteration was run for 1000 iterations with a discount factor of $\gamma = 0.99$ to ensure a focus on distant time horizons.

The proposed method was not run for this system, as it relies inherently on uneven singular values.

\subsubsection{Panel B}
Random sampling was completed through MersenneTwister in the Random library of Julia with the seed 1234, chosen arbitrarily.

The system followed the linear dynamics
\begin{equation}
\dot x = \begin{bmatrix}-10.0  & 0 \\ 0 & -0.1 \end{bmatrix}x,
\end{equation}
discretized into timesteps of duration 0.01.

As the system trajectory value does not impact the CRLB updates in this system, 100 PSD matrix samples were used with eigenvalues following exponential distributions with parameter $\lambda=1$, while only one state space sample was chosen.
The maximum distance for averaging was chosen as $d_{max}=2$.

Observations of the system were made under i.i.d. Gaussian noise with variance $\sigma^2 = 1$.
The action space for both the dynamic programming baseline approach and the proposed method was the set of vectors along the unit circle spaced by $0.05$ radians between $0$ and $\pi$.

The CRLB was initialized as
\begin{equation}
    \munderbar{\Sigma} = \begin{bmatrix} 4 & 0 \\
    0 & 4 \end{bmatrix},
\end{equation}
while the state was randomly initialized by sampling two i.i.d. Gaussian random variables with variance $4$.

Value iteration was run for 10000 iterations with a discount factor of $\gamma = 0.95$ to ensure a focus on distant time horizons.

For the proposed technique, the limiting singular vector was approximated at $10$ seconds of system time in the future.

\subsubsection{Panel C}
Random sampling was completed through MersenneTwister in the Random library of Julia with the seed 1234, chosen arbitrarily.

The system followed a Hopf bifurcation dynamic
\begin{equation}
\begin{split}
\dot x_{(1)} &= x_{(1)}(1 - \|x\|^2) - x_{(2)} \\
\dot x_{(2)} &= x_{(2)}(1 - \|x\|^2) + x_{(1)} 
\end{split}
\end{equation}
discretized into timesteps of duration $0.01$.

The dynamic programming space was approximated using the product space of 30 PSD matrix samples with eigenvalues following exponential distributions with parameter $\lambda=1$ and 120 state space samples.
The state space samples were chosen by selecting 30 i.i.d. Gaussian vectors with each element independent with variance 25, then advancing each of the points by 4 timesteps, retaining each sample in the trajectory.
The maximum distance for averaging was chosen as $d_{max}=2$.

Observations of the system were made under i.i.d. Gaussian noise with variance $\sigma^2 = 1$.
The action space for both the dynamic programming baseline approach and the proposed method was the set of vectors along the unit circle spaced by $0.1$ radians between $0$ and $\pi$.

The CRLB was initialized as
\begin{equation}
    \munderbar{\Sigma} = \begin{bmatrix} 4 & 0 \\
    0 & 4 \end{bmatrix},
\end{equation}
while the state was randomly initialized by sampling two i.i.d. Gaussian random variables with variance $4$.

Value iteration was run for 100 iterations with a discount factor of $\gamma = 0.99$ to ensure a focus on distant time horizons.

For the proposed technique, the limiting singular vector was approximated at $30$ seconds of system time in the future.

\subsubsection{Panel D}
Random sampling was completed through MersenneTwister in the Random library of Julia with the seed 1234, chosen arbitrarily.

The system followed a Hopf bifurcation dynamic
\begin{equation}
\begin{split}
\dot x_{(1)} &= 10(x_{(2)} - x_{(1)})\\
\dot x_{(2)} &= x_{(1)}(28 - x_{(3)}) - x_{(2)} \\
\dot x_{(3)} &= x_{(1)}x_{(2)} - \frac{8}{3}x_{(3)}
\end{split}
\end{equation}
discretized into timesteps of duration $0.01$.

The dynamic programming space was approximated using the product space of 30 PSD matrix samples with eigenvalues following exponential distributions with parameter $\lambda=1$ and 120 state space samples.
The state space samples were chosen by selecting 100 i.i.d. Gaussian vectors with each element independent with variance 25.
The maximum distance for averaging was chosen as $d_{max}=2$.

Observations of the system were made under i.i.d. Gaussian noise with variance $\sigma^2 = 1$.
The action space for both the dynamic programming baseline approach and the proposed method was a set of 50 i.i.d. uniformly chosen unit vectors.

The CRLB was initialized as
\begin{equation}
    \munderbar{\Sigma} = \begin{bmatrix} 1 & 0 & 0 \\
    0 & 1 & 0 \\
    0 & 0 & 1\end{bmatrix},
\end{equation}
while the state was randomly initialized by sampling two i.i.d. Gaussian random variables with variance $4$.

Value iteration was run for 100 iterations with a discount factor of $\gamma = 0.9$, intentionally lower due to the chaotic behavior of the system.

For the proposed technique, the limiting singular vector was approximated at $1$ second of system time in the future, lower than the others to account for the lack of actual convergence.
The time horizon was chosen based upon the time horizon being plotted.

\subsection{Figure \ref{fig:scaling}}

\subsubsection{Panels A and B}
Simulations were identical to those in Figure \ref{fig:SystemSampling} with the additional step of the extended Kalman filter.
Observation decisions were made according to the estimated state.

The cost was normalized to the oracle solution for the dynamic programming approach in order to emphasize the difference between the performance of the four techniques.

\subsubsection{Panel C}
Random sampling was completed through MersenneTwister in the Random library of Julia with the seed 1234, chosen arbitrarily.

The simulated system followed a Van Der Pol Oscillator with $\mu=1$ as the chosen parameter augmented with additional linear dimensions governed by $\dot{x}_{(i)} = -x_{(i)}$.

The CRLB was evaluated after 1000 observations, spaced by $0.05$ units of system time.
Observations were made under i.i.d. Gaussian noise with variance $\sigma^2 = 1$.

1000 steps of gradient ascent were completed for each observation with decaying step sizes $\alpha_i = 10^4 i^{-2/3}$, which empirically helped with ensuring quick convergence for all objective functions.

\subsection{Figure \ref{fig:distances}}
Random sampling was completed through MersenneTwister in the Random library of Julia with the seed 4321, chosen arbitrarily.

For each system dimensionality $M$, 10000 i.i.d. $M \times M$ positive definite matrices were generated with eigenvalues sampled from an exponential distribution with parameter $\lambda = 1$.

Then, for each given number of samples $N$, we choose $N+1$ points uniformly without replacement from the original dataset and compute the minimum distance from one of those selected points chosen randomly from the subset.
The process is averaged 5000 times.

\end{document}